\documentclass{article}

\usepackage[final, nonatbib]{neurips_2022}

\usepackage[utf8]{inputenc} 
\usepackage[T1]{fontenc}    
\usepackage{hyperref}       
\usepackage{url}            
\usepackage{booktabs}       
\usepackage{amsfonts}       
\usepackage{nicefrac}       
\usepackage{microtype}      
\usepackage{xcolor}         

\usepackage{algorithmic}
\usepackage{algorithm}
\usepackage{multicol}
\usepackage{amsmath}
\usepackage{amssymb}
\usepackage{mathtools}
\usepackage{amsthm}
\usepackage{bbm, xcolor}

\DeclareMathOperator*{\argmax}{arg\,max}
\DeclareMathOperator*{\argmin}{arg\,min}

\newtheorem{theorem}{Theorem}
\newtheorem{lemma}[theorem]{Lemma}
\newtheorem{proposition}[theorem]{Proposition}

\newtheorem{definition}[theorem]{Definition}
\newtheorem{assumption}[theorem]{Assumption}

\newcommand{\EE}{\mathbb{E}}

\newcommand{\NN}{\mathbb{N}}

\newcommand{\RR}{\mathbb{R}}

\newcommand{\cA}{\mathcal{A}}

\newcommand{\cC}{\mathcal{C}}

\newcommand{\cF}{\mathcal{F}}
\newcommand{\cG}{\mathcal{G}}
\newcommand{\cH}{\mathcal{H}}

\newcommand{\cL}{\mathcal{L}}
\newcommand{\cM}{\mathcal{M}}
\newcommand{\cN}{\mathcal{N}}

\newcommand{\cR}{\mathcal{R}}

\newcommand{\cT}{\mathcal{T}}
\newcommand{\cU}{\mathcal{U}}
\newcommand{\cV}{\mathcal{V}}

\newcommand{\cX}{\mathcal{X}}
\newcommand{\cY}{\mathcal{Y}}

\DeclarePairedDelimiter\abs{\lvert}{\rvert}%
\DeclarePairedDelimiter\norm{\lVert}{\rVert}%

\DeclarePairedDelimiter\autoparens{(}{)}
\newcommand{\prs}[1]{\autoparens*{#1}}
\DeclarePairedDelimiter\autobrackets{[}{]}
\newcommand{\brs}[1]{\autobrackets*{#1}}
\DeclarePairedDelimiter\autocurlybrackets{\{}{\}}
\newcommand{\cbs}[1]{\autocurlybrackets*{#1}}
\DeclarePairedDelimiter\autoinnerproduct{\langle}{\rangle}
\newcommand{\innerprod}[1]{\autoinnerproduct*{#1}}
\makeatletter
\let\oldabs\abs
\def\abs{\@ifstar{\oldabs}{\oldabs*}}
\let\oldnorm\norm
\def\norm{\@ifstar{\oldnorm}{\oldnorm*}}
\makeatother

\usepackage{multirow}
\usepackage{float}
\mathtoolsset{showonlyrefs=true}
\usepackage[maxbibnames=9,maxcitenames=2,backend=biber]{biblatex}
\title{Learning from Label Proportions by Learning with Label Noise}

\author{%
Jianxin Zhang, Yutong Wang, and Clayton Scott\\
  Electrical Engineering and Computer Science \\
  University of Michigan \\
  Ann Arbor, MI 48109 \\
  \texttt{\{jianxinz, yutongw, clayscot\}@umich.edu} \\
}

\begin{document}

\maketitle

\begin{abstract}
    Learning from label proportions (LLP) is a weakly supervised classification problem where data points are grouped into bags, and the label proportions within each bag are observed instead of the instance-level labels. The task is to learn a classifier to predict the labels of future individual instances. Prior work on LLP for multi-class data has yet to develop a theoretically grounded algorithm. In this work, we propose an approach to LLP based on a reduction to learning with label noise, using the forward correction (FC) loss of \textcite{Patrini2017MakingDN}. We establish an excess risk bound and generalization error analysis for our approach, while also extending the theory of the FC loss which may be of independent interest. Our approach demonstrates improved empirical performance in deep learning scenarios across multiple datasets and architectures, compared to the leading methods. 
\end{abstract}

\section{Introduction}
In the weakly supervised problem of \emph{learning from label proportions} (LLP), the learner is presented with bags of instances, where each bag is annotated with the proportions of the different classes in the bag. The learner's objective is to produce a classifier that accurately assigns labels to individual instances in the future. LLP arises in various applications including high energy physics \cite{Dery2018WeaklySupervised}, election prediction \cite{Sun2017ElectionPred}, computer vision \cite{Chen2014ObjectBased, Lai2014VideoEvent}, medical image analysis \cite{Gerda2018LLP4Emphysema}, remote sensing \cite{Ding2017LearningFL}, activity recognition \cite{poyiadzi18}, and reproductive medicine \cite{hernandez18}.

To date, most methods for LLP have addressed the setting of binary classification \cite{Yu13PSVM, Rueping10invcal, binary, Quadrianto08, SHI18LLPHD, Stolpe11LLPcluster, pmlr-v139-lu21c, saket2021learnability, saket22aistats}, although multiclass methods have also recently been investigated \cite{DulacArnold2019DeepML,Liu19LLPGAN,Tsai20LLPVAT}. The dominant approach to LLP in the literature is ``label proportion matching'': train a classifier to accurately reproduce the observed label proportions on the training data, perhaps with additional regularization. In the multiclass setting, the Kullback-Leibler (KL) divergence between the observed and predicted label proportions is adopted by the leading approaches to assess proportion matching. Unfortunately, while matching the observed label proportions is intuitive and can work well in some settings, it has little theoretical basis \cite{Yu13PSVM,saket22aistats}, especially in the multiclass setting, and there are natural settings where it fails \cite{Yu13PSVM, binary}. 

Recently, \textcite{binary} demonstrated a principled approach to LLP with performance guarantees based on a reduction to learning with label noise (LLN) in the binary setting. Their basic strategy was to pair bags, and view each pair of bags as an LLN problem, where the observed label proportions are related to the ``label flipping'' or ``noise transition'' probabilities. Using an existing technique for LLN based on loss correction, which allows the learner to train directly on the noisy data, they formulated an overall objective based on a (weighted) sum of objectives for each pair of bags. They established generalization error analysis and consistency for the method, and also showed that in the context of kernel methods, their approach outperformed the leading kernel methods.

The objective of the present paper is to develop a theoretically grounded and practical approach to multiclass LLP, drawing inspiration from \textcite{binary}. The primary challenge stems from the fact that \textcite{binary} employed the so-called ``backward correction'' loss, which solves LLN by scaling the \emph{output} of a loss function of interest according to the noise transition probabilities \cite{Natarajan2013LNL, Patrini2017MakingDN, rooyen18jmlr}. While this loss correction was demonstrated to work well for kernel methods in a binary setting, \textcite{Patrini2017MakingDN} introduced an alternative loss correction that performs better empirically in deep learning settings (see also \cite{pmlr-v139-zhang21k}). They proposed the ``forward correction'' loss, which scales the \emph{inputs} to a loss function of interest according to the noise transition probabilities. \textcite{Patrini2017MakingDN} find that backward correction ``does not seem to work well in the low noise regime,'' and is ``a linear combination of losses'' with ``coefficients that can be far [apart] by orders of magnitude '' which ``makes the learning harder''.

The present work is thus inspired by \textcite{binary} but uses the forward correction (FC) loss in a multiclass setting. This requires a number of technical modifications to the arguments of \textcite{binary}. Most notably, it now becomes necessary to demonstrate that the FC loss is \emph{calibrated} with respect to the 0-1 loss, a critical property needed for showing consistency. Such analysis is inherently not needed when using the backward correction, where the target excess risk is \emph{proportional to} the surrogate excess risk (from which calibration follows trivially). Furthermore, \textcite{binary} does not require analysis of proper composite losses, which are needed in the FC framework. Finally, the multiclass setting involves new estimation challenges not present in the binary case. These factors mean that our work is not a straightforward extension of \textcite{binary}. Indeed, the authors of a recent report acknowledge that it is ``difficult to extend [the method of \textcite{binary}] to multiclass classification" \cite{kobayashi22arxiv}.

{\bf Additional related work:} Much work on LLP has focused on learning specific types of models, including support vector machines \cite{Rueping10invcal, Yu13PSVM, wang15, qi17, chen17npsvm, lai14cvpr, shi17}, probabilistic models \cite{kuck05uai, hernandez13, Sun2017ElectionPred, poyiadzi18, hernandez18}, random forests \cite{SHI18LLPHD}, neural networks \cite{li15alter, ardehaly17, DulacArnold2019DeepML, Liu19LLPGAN, Tsai20LLPVAT}, and clustering-based models, \cite{chen09, Stolpe11LLPcluster}. Many of these works develop learning criteria that are specific to the model being learned. 

On the theoretical front, \textcite{Quadrianto08} and \textcite{Patrini2017MakingDN} initiated the learning theoretic study of LLP, introducing Rademacher style bounds for linear methods, but they do not address consistency w.r.t. a classification performance measure. \textcite{yu15tr} provides support for label proportion matching but only under the assumption that the bags are very pure. \textcite{saket2021learnability} studies learnability of linear threshold functions. Recently \textcite{saket22aistats} introduced a condition under which label proportion matching does provably well w.r.t. a squared error loss in the binary setting, and developed an associated algorithm. This method does not scale easily to large datasets, and further requires knowledge of how bags are grouped according to different bag-generating distributions.

A handful of recent papers have studied multiclass LLP in deep learning scenarios. 
\textcite{DulacArnold2019DeepML} study the KL loss for label proportion matching, and a variant based on optimal transport. \textcite{Liu19LLPGAN,Liu21LLPGANPLOT} examine an approach based on generative adversarial models. \textcite{Tsai20LLPVAT} study the use of a regularizer derived from semi-supervised learning. One challenge common to these approaches is that their implementations employ mini-batches of bags, which becomes computationally prohibitive for large bag sizes when the batch size is still very small, e.g., 2 or 3 bags. In contrast, our approach avoids this issue. 
Finally, a recent technical report presents a risk analysis for multiclass LLP under the assumption of fixed bag size, which we do not require \cite{kobayashi22arxiv}. Their method is not tractable for large bag sizes in which case they approximate their objective ``using the bag-level loss proposed in the existing research."

{\bf Contributions and Outline:} Our contributions and the paper structure are summarized as follows. In Section \ref{llnfc}, we review the FC loss as a solution to LLN. In Section \ref{sec:calfun}, we extend the theory of the FC loss for LLN. In particular, we show that the FC loss is ``uniformly calibrated'' with respect to the 0-1 loss using the framework of \textcite{Steinwart2007HowTC}, establish an excess risk bound, and determine an explicit lower bound on the calibration function in terms of the noise transition matrix.  In Section \ref{sec:mntm}, we extend the results of Section \ref{sec:calfun} to the setting with multiple noise transition matrices, which form the basis of our approach to LLP. In particular, we establish an excess risk bound and generalization error analysis for learning with multiple noise transition matrices, which in turn enables proofs of consistency. In Section \ref{sec:llpfc}, we state the probabilistic model for reducing LLP to LLN with multiple different noise transition matrices and present the LLPFC algorithms. Experiments with deep neural networks are presented in Section \ref{sec:exp}, where we observe that our approach outperforms competing methods by a substantial margin. Proofs appear in the supplemental material.

\section{Learning with Label Noise and the Forward Correction Loss} \label{llnfc}

This section sets notation 
and introduces the FC loss as a solution to learning with label noise. Let $\cX$ be the feature space and $\cY = \{1, 2, \dots, C \}$ be the label space, $C\in \mathbb{N}$. We define the $C$-simplex as $\Delta^C = \{ p \in \mathbb{R}^C: p_i \geq 0, \forall i = 1, 2, \dots, C, \sum_{i=1}^C p_i = 1\}$ and denote its interior by $\mathring{\Delta}^C$. Let $P$ be a probability measure on the space $\cX \times \cY$.

Viewing $P$ as the ``clean'' probability measure, a noisy probability measure with label-dependent label noise can be constructed from $P$ in terms of a $C \times C$ column-stochastic matrix $T$, referred to as the \emph{noise transition matrix}. Formally, we define a measure $\Bar{P}_T$ on $\cX \times \cY \times \cY$ by requiring
$\forall \text{ events } \cA \subset \cX, \Bar{P}_T(\cA \times \{i\} \times \{j\}) = P(\cA \times \{i\})t_{j,i}$  where $t_{j,i}$ is the element at the $j$-th row and $i$-th column of $T$. Let $(X, Y, \Tilde{Y})$ have joint distribution $\Bar{P}_T$ where $X$ is the feature vector, $Y$ is the ``clean'' label, and $\Tilde{Y}$ is the ``noisy'' label. Thus the element of $T$ at row $i$ and column $j$ is $t_{i,j} = \Bar{P}_T (\tilde{Y}=i | Y=j)$. In addition, $P$ is the marginal distribution of $(X, Y)$. Define  $P_T$ to be the marginal distribution of $(X, \Tilde{Y})$. Let $\mathcal{F}$ be the collection of all measurable functions from $\cX$ to $\Delta^C$.

The existence of a regular conditional distribution is guaranteed by the Disintegration Theorem ($e.g.$ Theorem 6.4 in \textcite{kallenberg_modern_prob}) under suitable properties ($e.g.$ when $\cX$ is a Radon space). While the existence of regular conditional probability is beyond the scope of this paper, we assume fixed regular conditional distributions for $Y$ and $\Tilde{Y}$ given $X$ exist, denoted by ${P}(\cdot\mid\cdot): \cY \times \cX \rightarrow \brs{0, 1}$ and ${P}_{T}(\cdot\mid\cdot): \cY \times \cX \rightarrow \brs{0, 1}$, respectively. Given $x\in \cX$, we define the probability vectors $\eta(x) = \brs{{P}(1\mid x), \dots, {P}(C\mid x)}^{tr}$ and $\eta_T(x) = \brs{{P_T}(1\mid x), \dots, {P_T}(C\mid x)}^{tr}$ where we use $tr$ to denote transposition.
It directly follows that $\eta_T(x) = T\eta(x)$.

We use $\mathbb{R}_+$ to denote the positive real numbers. The goal of LLN is to learn a classifier that optimizes a performance measure defined $w.r.t.$ $P$, given access to corrupted training data $(X_i, \Tilde{Y_i} ) \stackrel{i.i.d.}{\sim} P_T$. In this work we assume $T$ is known or can be estimated, as is the case when we apply LLN techniques to LLP (see Section \ref{sec:llpfc}). A more formal formulation of LLP is given in Section \ref{sec:llpfc}.

When attempting to minimize the risk associated to the 0-1 loss and the clean distribution $P$, it is common to employ a smooth or convex surrogate loss. For LLN problems, the idea of a {\em loss correction} is to modify the surrogate loss so that when optimized using the \emph{noisy} data, it still achieves the desired goal. Below, we introduce the forward correction loss, before which we need to define inner risk and proper loss. For this purpose we focus on loss functions of the form $L: \Delta^C \times \cY \rightarrow \mathbb{R}$.



\begin{definition}
Let $L: \Delta^C \times \cY \rightarrow \mathbb{R}$ be a loss function. The \textbf{inner $L$-risk} at $x$ with probability measure $P$ is $\cC_{L, P, x} : \Delta^C \rightarrow \mathbb{R} , \ \ \cC_{L, P, x} (q) := \mathbb{E}_{Y \sim P( \cdot | x)} L(q, Y).$ The \textbf{minimal inner $L$-risk} at $x$ with a probability measure $P$ is $\cC_{L, P, x}^{*} := \inf_{q \in \Delta^C} \cC_{L, P, x} (q).$
\end{definition}

\begin{definition}
$\ell: \Delta^C \times \cY \rightarrow \mathbb{R}$ is a \textbf{proper loss} if $\forall$ probability measures $P$ on $\cX \times \cY$, $\forall x \in \cX, \cC_{\ell, P, x}^{*} = \cC_{\ell, P, x} (\eta (x)),$ and a proper loss is called \textbf{strictly proper} if the minimizer of $\cC_{\ell, P, x}$ is unique for all $x\in\cX$.
\end{definition}

Commonly used proper losses include the \textit{log loss} $\ell^{log}\prs{q, c} = - \log q_c$, the \textit{square loss} $\ell^{sq}\prs{q, c}=\sum_{c'=1}^C \prs{\mathbbm{1}_{c=c'} - q_{c'}}^2$, and the \textit{$0$-$1$ loss} 
$\ell^{01}\prs{q, c} = \mathbbm{1}_{c \ne \min\{\argmax_j q_j\}}$, 
among which only the log loss and the square loss are strictly proper \cite{Williamson2016ProperComp}. Here $\mathbbm{1}$ denotes the indicator function. Note that it is common to compose proper losses with inverted link functions, leading to familiar losses like the cross-entropy loss. Such losses are discussed further in Section \ref{sec:mntm}.


We are now ready to introduce the forward correction loss. 
\begin{definition}
Let $\ell$ be a strictly proper loss and let $T$ be a noise transition matrix. Define the \textbf{forward correction loss} of $\ell$ as $\ell_{T}: \Delta^C \times \cY \rightarrow \mathbb{R} , \ \ \ell_{T}(q, c) := \ell(Tq, c).$
\end{definition}

It follows from the definition that, if $T$ is invertible, then the inner $\ell_{T}$-risk under the distribution $P_T$ has a unique minimizer $\eta(x)$. Next we introduce $L$-risk and $L$-Bayes risk associated with a loss $L$. 

\begin{definition}
Let $L: \Delta^C \times \cY \rightarrow \mathbb{R}$ and $P$ be a probability measure. Define the \textbf{$L$-risk} of $f$ with distribution $P$ to be $\mathcal{R}_{L, P}: \mathcal{F} \rightarrow \mathbb{R} , \ \ \mathcal{R}_{L, P}(f):= \mathbb{E}_{P} \left[ L(f(X), Y) \right]$ and the \textbf{$L$-Bayes risk} to be $\mathcal{R}_{L, P}^* := \inf_{f\in \mathcal{F}} \mathcal{R}_{L, P}(f).$
\end{definition}

We call $\mathcal{R}_{L, P}(f) - \mathcal{R}_{L, P}^*$ the \textit{excess $L$-risk} of $f$ under distribution $P$. Given a proper loss $\ell$, Theorem 2 of \textcite{Patrini2017MakingDN} establishes Fisher consistency of the FC loss, meaning the minimizer of $\ell$-risk under the clean distribution $P$ is the same as the minimizer of $\ell_{T}$-risk under noisy distribution $P_T$: $\argmin_{f \in \cF} \cR_{L, P}(f) = \argmin_{f \in \cF} \cR_{\ell_{T}, P_T}(f)$. Next, we present a stronger result relating the excess $\ell_{T}$-risk under the noisy distribution $P_T$ to the excess 0-1 risk under the clean distribution $P$.

\section{Calibration Analysis for the Forward Correction Loss}
\label{sec:calfun}

Our objective in this section is to show that when $L$ is the 0-1 loss and $\ell$ is a continuous strictly proper surrogate loss, there exists a strictly increasing, invertible function $\theta$ with $\theta(0) = 0$ such that $\forall f \in \cF$ and $\forall$ distributions $P$, $\theta \prs{\cR_{L, P}(f) - \cR^*_{L, P} } \leq  \cR_{\ell_{T}, P_T}(f) - \cR_{\ell_{T}, P_T}^*$. Given such a bound, it follows that consistency 
$w.r.t$ the surrogate risk implies consistency $w.r.t.$ the target risk. The results in this section are standalone results for the FC loss that may be of independent interest, and will be extended in the next section in relation to LLP. 
The following theorem guarantees the existence of such function $\theta$, given that 
$T$ is invertible.

\begin{theorem} \label{thm:inv.cali}
Let $\ell$ be a continuous strictly proper loss and T be an invertible column-stochastic matrix. Let $L$ be the $0$-$1$ loss. Assume $\cR_{\ell_{T}, P_T}^* < \infty$. Then $\exists \theta: \brs{0, 1} \rightarrow \brs{0, \infty}$ that is strictly increasing and continuous, satisfying $\theta(0) = 0$, such that $\forall f \in \cF,  \cR_{L, P}(f) - \cR^*_{L, P} \leq \theta^{-1} \prs{\cR_{\ell_{T}, P_T}(f) - \cR_{\ell_{T}, P_T}^*}.$
\end{theorem}

The function $\theta$ in Theorem \ref{thm:inv.cali} depends on $\ell$ and $T$. The following proposition provides a convex lower bound on $\theta$ for the commonly used log loss $\ell^{log}\prs{q, c} = - \log q_c$. 
  Let $M \in \mathbb{R}^{C \times C}$ be a matrix and let $\| \cdot \|$ be a norm on $\mathbb{R}^C$.
  The \emph{subordinate matrix norm} induced by $\|\cdot \|$ is 
  $\|M\| := \sup_{x \in \mathbb{R}^C: x \ne 0} \frac{\|Mx\|}{\|x\|}$.
  When $\|\cdot \|$ is the $1$-norm on $\mathbb{R}^C$, the induced norm is denoted $\|M\|_1$, referred to as the matrix 1-norm, and can be computed as
$\|M\|_1 = \max_{1 \le j \le C} \sum_{i=1}^C |M(i,j)|$
  \cite{gallierlinear}.
  
\begin{proposition}
  \label{theorem: cali_lowerb}
  Let $T \in \mathbb{R}^{C\times C}$ be an invertible, column-stochastic matrix. Define $\underline{\theta}_T : \brs{0, \infty} \rightarrow \brs{0, \infty}$ by $ \underline{\theta}_T(\epsilon) = \frac{1}{2} \frac{\epsilon^2} {\|T^{-1}\|_1^2}.$
  If $L$ is the 0/1 loss, $\ell$ is the log loss, then 
  for all $f \in \cF$ and distributions $P$,
  $
      \cR_{L, P}(f) - \cR^*_{L, P} 
      \leq \underline{\theta}_T^{-1} \prs{\cR_{\ell_{T}, P_T}(f) - \cR_{\ell_{T}, P_T}^*}
      = \sqrt{2} \|T^{-1}\|_1 \sqrt{\cR_{\ell_{T}, P_T}(f) - \cR_{\ell_{T}, P_T}^*}
  $
\end{proposition}
The factor $\|T^{-1}\|_1$ may be viewed as a constant that captures the overall amount of label noise. The more noise, the larger the constant. For example, let $I$ and $N$ be the identity and the all $1/C$'s matrices, respectively. Let $\alpha \in [0,1]$ and  $T = (1-\alpha) I + \alpha N$. Thus, $\alpha = 0$ represents the noise-free case and $\alpha = 1$ the noise-only case. It is easy to verify that $T^{-1} = (1-\alpha)^{-1}(I - \alpha N)$ and $\|T^{-1}\|_1 = (1-\alpha)^{-1}(1+(1-2/C)\alpha)$. 

\section{Learning with Multiple Noise Transition Matrices}
\label{sec:mntm}

Our algorithms for LLP, formally stated in subsection \ref{subsec: llpfc algo}, reduce the problem of LLP to LLN by partitioning bags into groups and modeling each group as an LLN problem. Since each group has its own noise transition matrix, this leads to a new problem that we refer to as learning with multiple noise transition matrices (LMNTM). In this section, we show how to extend the calibration analysis of section \ref{sec:calfun} to this setting. In addition, we offer a generalization error bound that justifies an empirical risk minimization learning procedure based on a weighted sum of FC losses.

\subsection{Learning with Multiple Noise Transition Matrices}
\label{subset: lmntm_setup}

We first define the LMNTM problem formally. For all $n \in \NN$, denote $\NN_n = \cbs{1, 2, \dots, n}$. Consider a clean distribution $P$ on $\cX \times \cY$ and noise transition matrices $T_1, T_2, \dots, T_N$. For each $i$ we denote the noisy prior as the $\alpha_i \in \mathring{\Delta}^C$ where, $\forall c \in \cY$, $\alpha_i(c) = P_{T_i}(\Tilde{Y}= c)$. We assume the $\alpha_i$'s are known for theoretical analysis. In practice, $\alpha_i$ is estimable as discussed below. In LMNTM, we observe data points $S = \cbs{X_{i, c, j} : \ \ i \in \NN_N, \ \ c \in \cY, \ \ j \in \NN_{n_{i,c}}}$ where $X_{i, c, j} \stackrel{iid}{\sim} P_{T_i}(\cdot\mid c)$, and $n_{i,c} \in \NN$ is the number of data points drawn from the class conditional distribution $P_{T_i}(\cdot\mid c)$. Assume all $X_{i,c, j}$'s are mutually independent. We make additional remarks on this setting in Section \ref{subsec:remarks lmntm} in the appendix.

\subsection{A Risk for LMNTM} \label{subsec: LMNTM}

The following result extends Theorem \ref{thm:inv.cali} to LMNTM. It establishes that the risk $\widetilde{R}_{\ell, P, \cT}$, which can be estimated from LMNTM training data, is a valid surrogate risk. This type of result is not needed for the backward correction approach of \textcite{binary}.

\begin{theorem} \label{llp.cali}
Let $L$ be the $0$-$1$ loss and $N \in \NN$. Consider a sequence of invertible column-stochastic matrices $\cT = \cbs{T_i}_{i=1}^{N}$ and a continuous strictly proper loss function $\ell$.  Let $w = \prs{w_i}_{i=1}^{N} \in \Delta^N$. Define 
$\widetilde{R}_{\ell, P, \cT}: \cF \rightarrow \RR \text{ by } \widetilde{R}_{\ell, P, \cT} \prs{f} :=  \sum_{i=1}^{N} w_i\cR_{\ell_{T_i}, P_{T_i}} \prs{f} $
and $\widetilde{R}_{\ell, P, \cT}^* = \inf_{f\in \cF} \widetilde{R}_{\ell, P, \cT} \prs{f}$. Assume $\forall i \in \cbs{1, 2, \dots, N}, \cR_{\ell_{T_i}, P_{T_i}}^* < \infty.$ Then $\exists$ a strictly increasing continuous function $\theta: \brs{0, 1} \rightarrow \brs{0, \infty}$ with $\theta \prs{0} = 0 \ \ s.t. $ for all $P$, $\forall f \in \cF$, $\theta \prs{\cR_{L, P}(f) - \cR_{L, P}^*} \leq \widetilde{\cR}_{\ell, P, \cT}(f) - \widetilde{\cR}_{\ell, P, \cT}^*.$
\end{theorem}

The weights $w_i$ allow the user flexibility, for example, to place different weights on noisier or larger subsets of data. Unlike \textcite{binary}, however, because the weights appear in both our excess risk bound and generalization error bound, it is not straightforward to optimize them a priori. We discuss weight optimization in detail in Section \ref{suppsec:weights} in the appendix.

\subsection{Generalization Error Bound} \label{subsec: geb}

The aggregate risk $\widetilde{R}_{\ell, P, \cT}$ is desirable because it can naturally be estimated from the given data. We propose the empirical risk 

\begin{equation} \label{llpfc-ideal: empirical risk}
    \hat{\cR}_{w, S}(f) = \sum_{i=1}^N w_i \sum_{c=1}^{C} \frac{\alpha_i(c)}{n_{i,c}} \sum_{j=1}^{n_{i,c}}\ {\ell_{T_i}} \prs{f\prs{X_{i,c,j}}, c}.
\end{equation}
It should be noted that $\hat{\cR}_{w, S}(f)$ is an unbiased estimate of  $\Tilde{\cR}_{\ell, P, \cT} (f)$. 
Here we establish a generalization error bound for this estimate which builds on Rademacher complexity analysis .

To state the bound, we must first introduce the notion of a \textit{proper composite loss} \cite{Williamson2016ProperComp}. This stems from the fact that in practice, a function $f$ outputting values in $\Delta^C$ is typically obtained by composing a ${\mathbb R}^C$-valued function (such as a neural network with $C$ output layer nodes), with another function ${\mathbb{R}^C}\to \Delta^C$ such as the softmax function. Thus, let $\psi: \cU \subset \Delta^C \rightarrow \cV$ be an invertible function where $\cV$ is a subset of a normed space, referred to as an \emph{invertible link function}. Consider $\cG \subset \psi \circ \cF := \cbs{\psi \circ f: f \in \cF}$, and observe that $\forall g \in \cG, \psi^{-1} \circ g \in \cF$. In practice, $\psi$ is fixed and we seek to learn $g \in \cG$ that leads to an $f \in \cF$ with a risk close to the Bayes risk. An example of $\psi^{-1}$ is the softmax function so that $\psi: \cU \rightarrow \cV, \psi_i(p) = \log p_i - \frac{1}{C}\sum_{k=1}^C \log p_k, {(\psi^{-1})}_i(s) = \frac{e^{s_i}}{\sum_{k=1}^C e^{s_k}}$ where $\cU$ is the interior of $\Delta^C$ and $\cV = \{s \in \RR^C: \sum_{i=1}^C s_i = 0\}$. This motivates the following definition. 
\begin{definition}Given an invertible link function $\psi: \cU \subset \Delta^C \rightarrow \cV$, we define the \textbf{proper composite loss} $\lambda_{\ell}$ of a proper loss $\ell: \Delta^C \times \cY \rightarrow \RR$ to be $\lambda_{\ell}: \cV \times \cY \rightarrow \RR, \ \ \lambda_{\ell}\prs{v, c} = \ell\prs{\psi^{-1}(v), c}.$
\end{definition}

For example, when $\ell$ is the log loss and $\psi^{-1}$ is the softmax function, $\lambda_{\ell}$ is the cross-entropy (or multinomial logistic) loss. With this notation, we are now able to state our generalization error bound for LMNTM. We study two popular choices of function classes, the reproducing kernel Hilbert space (RKHS) and the multilayer perceptron (MLP). We use $\cG_1$ to denote the Cartesian product of $C$ balls of radius R in the RKHS and  $\cG_2$ to denote a multilayer perceptron with $C$ outputs. 
\begin{definition}
Let $k$ be a symmetric positive definite (SPD) kernel, and let $\mathcal{H}$ be the associated reproducing kernel Hilbert space (RKHS). Assume $k$ is bounded by $K$, meaning $\forall x$, $\norm{k(\cdot, x)}_{\cH} \leq K$. Let $\cG^k_{K,R}$ denote the ball of radius R in $\cH$. Define $\cG_1 = \cG^k_{K,R} \times \cG^k_{K,R} \times \dots \times \cG^k_{K,R}$ ($C$ copies).
\end{definition}
We follow \textcite{zhang2017radnn} and define real-valued MLPs inductively:
\begin{definition} 
Define $\cN_1 = \cbs{x\rightarrow \innerprod{x, v}: v\in \RR^d, \norm{v}_2 \leq \beta}$, and for $m > 2$, inductively define
$\cN_m = \cbs{ x\rightarrow \sum_{j=1}^d v_j \mu(f_j(x)): v\in \RR^d, \norm{v}_1 \leq \beta, f_j \in \cN_{m-1} }$,
where $\beta \in \RR_+$ and $\mu$ is a $1$-Lipschitz activation function. Define an MLP which outputs a vector in $\RR^C$ by $\cG_2 = \cN_m \times \cN_m \times \dots \times \cN_m$ ($C$ copies).
We additionally assume that the choice of $\mu$ satisfies $\forall m \in \NN, 0 \in \mu \circ \cN_m$.
\end{definition} 

\begin{theorem} \label{thm: geb}
Let $T_1$, $T_2$, $\dots$, $T_N$ be invertible column-stochastic matrices. Let $\ell$ be a proper loss such that $\forall i, c$ the function $\lambda_{\ell_{T_i}} \prs{\cdot ,c}$ is Lipschitz continuous $w.r.t.$ the $2$-norm. Let $S$ be the set of data points as defined in Section \ref{subset: lmntm_setup}. Assume $\sup_{x\in \cX, g \in \cG_q} \norm{g\prs{x}}_2 \leq A_q$ for some constant $A_q$, $\forall q \in \{1, 2\}$.  Let $\hat{\cR}_{w, S}$ be as defined in equation (\ref{llpfc-ideal: empirical risk}).
$ \widetilde{\cR}(g) := \widetilde{R}_{\ell, P, \cT} \prs{\psi^{-1} \circ g} = \EE_{S} \brs{\hat{\cR}_{w, S}(g)}.$ Then for each $q\in \cbs{1, 2}$, $\forall \delta \in \brs{0, 1}$, with probability at least $1-\delta$,
$$
    \sup_{g\in\cG_q} \abs{\hat{\cR}_{w, S}(g) - \widetilde{\cR}(g)}  \leq
    ( \max_i (\abs{ \lambda_{\ell_{T_i}}}A_{q} + \abs{ \lambda_{\ell_{T_i}}}_0) \sqrt{2 \log \frac{2}{\delta}} + C B_{q} \max_i\abs{ \lambda_{\ell_{T_i}}}) \sqrt{\sum_{i=1}^N \sum_{c=1}^C 
    \frac{{w_i^2}}{n_{i,c}}}.
$$
where $B_q$ is a constant depending on $\cG_q$, $\abs{ \lambda_{\ell_{T_i}}}_0 = \max_{c} \abs{ \lambda_{\ell_{T_i}}(0, c)}$, and $\abs{ \lambda_{\ell_{T_i}}}$ is the smallest real number such that it is a Lipschitz constant of $ \lambda_{\ell_{T_i}} \prs{\cdot ,c}$ for all $c$. 
\end{theorem}

Theorem \ref{thm: geb} is a special case of of Lemma \ref{lemma: gen} which extends the notion of Rademacher complexity to the LMNTM setting and applies to arbitrary function classes. Lemma \ref{lemma: gen} is presented in the appendix. 

Let $HM_i$ denote the harmonic mean of $n_{i,1}$, $\dots$, $n_{i,C}$, $i.e.$, $HM_i=\frac{C}{\sum_{c=1}^C {\frac{1}{n_{i,c}}}}$. The term $\sqrt{\sum_{i=1}^N \sum_{c=1}^C \frac{w_i^2}{n_{i,c}}}$ could be written as $\sqrt{C\sum_{i=1}^N \frac{w_i^2}{HM_i}}$ and is optimized by $w_i = HM_i/\sum_{m=1}^N {HM_m}$, leading to $\sqrt{\sum_{i=1}^N \sum_{c=1}^C \frac{w_i^2}{n_{i,c}}} = \sqrt{\frac{C}{\sum_{i=1}^N HM_i}}$. The term $\sqrt{\frac{C}{\sum_{i=1}^N HM_i}}$ vanishes (needed to establish consistency) when $N$ goes to infinity, or when $\exists i$ $s.t.$ $\forall c$, $n_{i,c}$ goes to infinity. For the special case where all bags have the same size $n$ and all weights $w_i$ are $1/N$, $\sqrt{\sum_{i=1}^N \sum_{c=1}^C \frac{w_i^2}{n_{i,c}}} =\sqrt{\frac{C}{Nn}}$. Thus, consistency is possible even if bag size remains bounded. Assuming $\ell$ is the log loss and $\psi^{-1}$ is the softmax function, we next study the constants $\abs{\lambda_{\ell_{T_i}}}$ and $\abs{\lambda_{\ell_{T_i}}}_0$.

\begin{proposition} \label{lip.upper.bd.cel}
Let $\ell$ be the log loss, $\psi^{-1}$ be the softmax function, and $T$ be a column-stochastic matrix. Then  $\abs{\lambda_{\ell_{T}}} \leq \sqrt{2}$.
\end{proposition}

The constant $\abs{\lambda_{\ell_{T} }}_0 = \max_{c} \abs{\lambda_{\ell_{T}}(0, c)}=\max_c -\log(\frac{1}{C}\sum_{j=1}^C t_{c,j})$. The invertibility of $T$ guarantees $\sum_{j=1}^C t_{c,j}$ is positive and hence the finiteness of $\abs{\lambda_{\ell_{T}}}_0$. However, if we have a ``bad" $T$, $\sum_{j=1}^C t_{c,j}$ could be arbitrarily close to $0$ leading to a large $\abs{\lambda_{\ell_{T}}}_0$.

Following Theorem \ref{thm: geb}, if the function class $\cG$ has a universal approximation property, such as an RKHS associated to a universal kernel, or an MLP with increasing number of nodes, consistency for LMNTM via (regularized) minimization of $\hat{\cR}_{w, S}(g)$ can be shown by leveraging standard techniques, provided $N \to \infty$ (bag size may remain bounded). Then the excess risk bound in Theorem \ref{llp.cali} would automatically imply consistency with respect to $0$-$1$ loss.

\section{The LLPFC algorithms} \label{sec:llpfc}
In this section, we define a probabilistic model for LLP, show how LLP reduces to LMNTM, and introduce  algorithms that we refer to as the LLPFC algorithms. 

\subsection{Probabilistic Model for LLP} \label{subsec:prob model}


Given a measure $P$ on the space $\cX \times \cY$, let $\cbs{P_c:c\in\cY}$ denote the class-conditional distributions of $\cX$, $i.e.$,  $\forall \text{ events } \cA \subset \cX, P_c \prs{\cA} = P\prs{\cA\mid Y=c}$. Let $\sigma\prs{c} = P(Y=c), \forall c\in \cY$ and call $\sigma = \prs{\sigma\prs{1}, \dots, \sigma\prs{C}}$ the clean prior. Assume $\forall c \in \cY, \sigma\prs{c} \neq 0$. Given $z = \prs{z\prs{1}, \dots, z\prs{C}} \in \Delta^C$, let $P_{z}$ be the probability measure on $\cX \times \cY$ s.t. $\forall \text{ events } \cA \subset \cX$, $\forall i \in \cY, P_{z} \prs{\cA \times \cbs{i}} = z\prs{i} P_i \prs{\cA}$. Thus $P_z$ has the same class-conditional distributions as $P$ but a variable prior $z$.

We first define a model for a single bag.
Given $z \in \Delta^C$, we say that bag $b$ is \textit{governed} by $z \in \Delta^C$ if $b$ is a collection of feature vectors $ \cbs{X_j: j\in \NN_{\abs{b}}}$ annotated by label proportion $\hat{z} = \prs{\hat{z}\prs{1},  \hat{z}\prs{2}, \dots, \hat{z}\prs{C}}$, where $\abs{b}$ denotes the cardinality of the bag, each $X_j$ is paired with an unobserved label $Y_j$ $s.t.$ $\prs{X_j, Y_j} \stackrel{iid}{\sim} P_{z}$, and $\hat{z}\prs{c} = \frac{1}{{\abs{b}}}{\sum_{j=1}^{\abs{b}}\mathbbm{1}_{Y_{j} = c}}$. Note $\EE_{P_z} \brs{\hat{z}} = z$ and $P_z \prs{Y_j= c} = z\prs{c}$. We think of $z$ as the true label proportion and $\hat{z}$ as the empirical label proportion.

Using this model for individual bags, we now formally state a model for LLP.
Given bags $\cbs{b_k}$, let each $b_k$ be governed by $\gamma_k$. Each $b_k$ is a collection of feature vectors $\cbs{X^k_j: j\in\NN_{\abs{b_k}}}$ where $(X^k_j, Y^k_j) \stackrel{i.i.d.}{\sim} P_{\gamma_k}$ and $ Y^k_j$ is unknown. 
Further assume the $X^k_{j}$'s are independent for all $k$ and $j$.
In practice, $\gamma_k$ is unknown and we observe $\hat{\gamma}_k$ with $\hat{\gamma}_k(c) = \frac{1}{\abs{b_k}}\sum_{j=1}^{\abs{b_k}}\mathbbm{1}_{Y^k_{j} = c}$ instead. The goal is learn an $f$ that minimizes the risk $\cR_{L, P} = \EE_{(X,Y)\sim P}\brs{L(f(X), Y)}$ where $L$ is the $0$-$1$ loss, given access to the training data $\{(b_k,\hat{\gamma}_k)\}$.


\subsection{The Case of C Bags: Reduction to LLN}
\label{subsec: prob model c bags}

To explain our reduction of LLP to LLN, we first consider the case of exactly $C$ bags $b_1, b_2, \dots, b_C$, governed by respective (unobserved) $\gamma_1, \ldots, \gamma_C \in \Delta^C$, and annotated with label proportions $\hat{\gamma}_1, \ldots \hat{\gamma}_C$.
Define $\Gamma \in \RR^{C\times C}$ by $\Gamma(i, j) = \gamma_i\prs{j}$, and let $\Gamma^{tr}$ denote the transpose of $\Gamma$. Recall that $\sigma$ is the class prior associated to $P$. To model LLP with $C$ bags as an LLN problem, we make the following assumption on $\Gamma$ and $\sigma$:

\begin{assumption} \label{grouping assumption: ch}
$\exists \text{ unique } \alpha \in \mathring{\Delta}^C$ $s.t.$ $ \Gamma^{tr} \alpha = \sigma$.
\end{assumption}

\begin{figure}\label{fig: simplex}
 \begin{minipage}[c]{0.45\textwidth}
   \includegraphics[width=80pt]{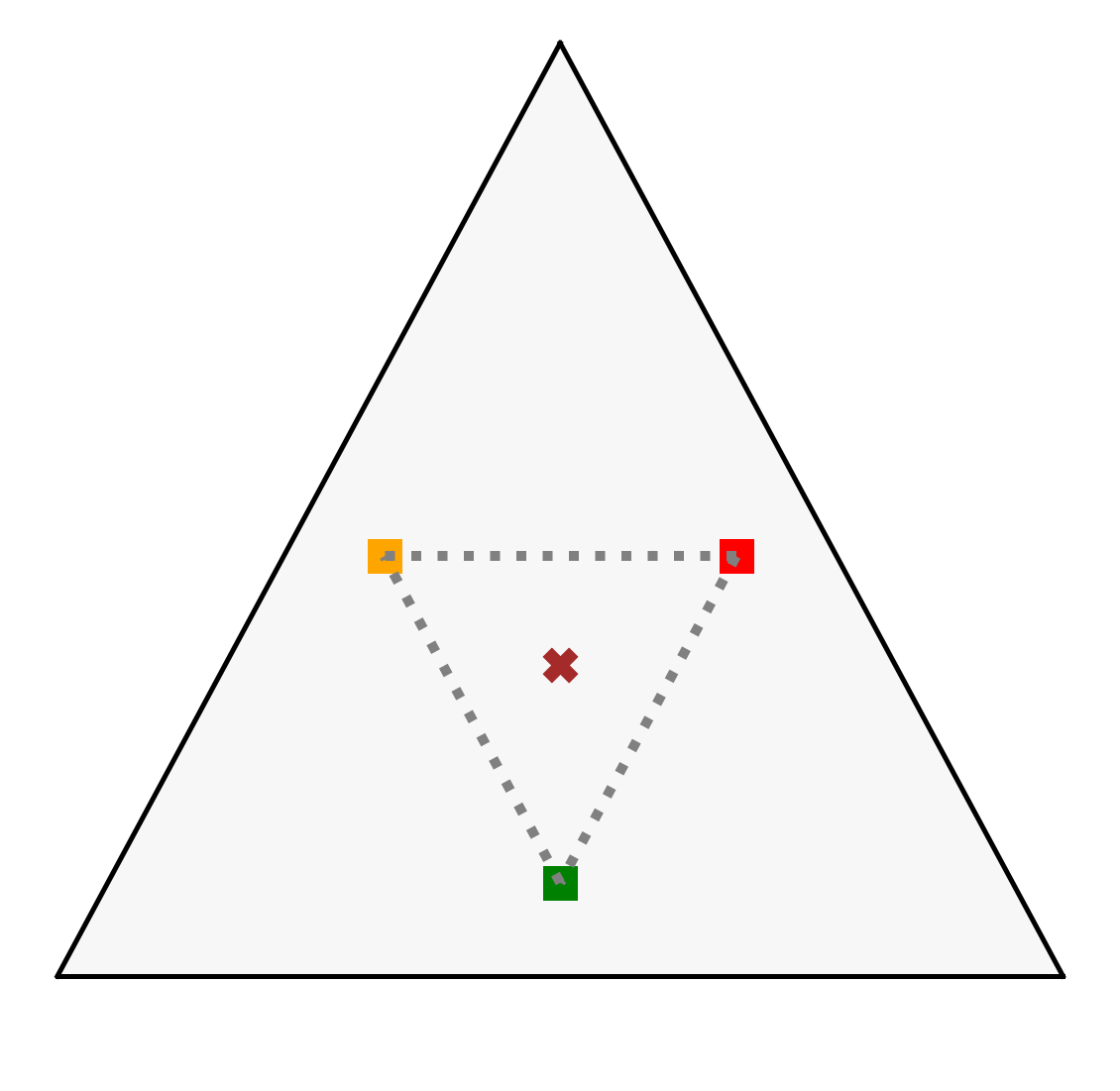}
\includegraphics[width=80pt]{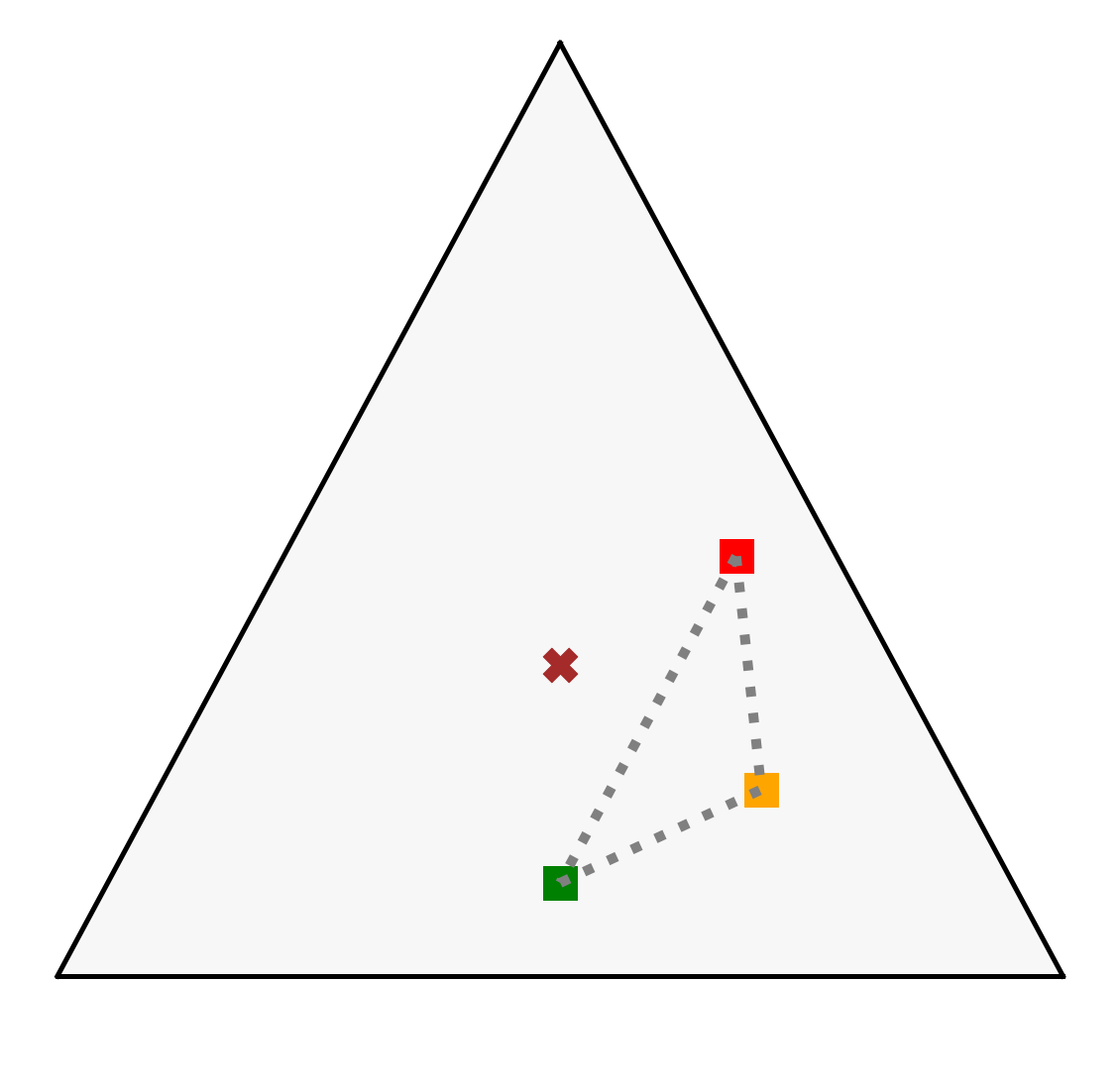}
  \end{minipage}\hfill
  \begin{minipage}[c]{0.55\textwidth}
    \caption{The gray triangle represents the probability simplex $\Delta^3$. The squares represent $\gamma_1$, $\gamma_2$, and $\gamma_3$. The cross is $\sigma$. The ternary graph on the left visualizes an example where Assumption \ref{grouping assumption: ch} holds. The one on the right visualizes an example where Assumption \ref{grouping assumption: ch} fails. }
  \end{minipage}
\end{figure}

 We write $\alpha = \prs{\alpha\prs{1}, \dots, \alpha\prs{C}}$. Assumption \ref{grouping assumption: ch} is equivalent to: $\cbs{\gamma_1, \dots, \gamma_C}$ is a linearly independent set and $\sigma$ is in the interior of the convex hull of $\cbs{\gamma_1, \dots, \gamma_C}$. Ternary plots in Figure \ref{fig: simplex} visualize examples where assumption \ref{grouping assumption: ch} holds and fails when $C=3$. Intuitively, assumption \ref{grouping assumption: ch} is more likely to hold when $\cbs{\gamma_i: i \in \NN_C}$ are more ``spread out'' in $\Delta^C$, in which case it is more likely for $\sigma$ to reside in the convex hull of $\cbs{\gamma_i: i \in \NN_C}$. 

To reduce LLP with $C$ bags to LLN, we simply propose to assign the ``noisy label'' $\tilde{Y} = i$ to all elements of bag $b_i$ and to construct a noise transition matrix $T$ with $T\prs{i,j} = \gamma_{i}\prs{j}\alpha\prs{i}/\sigma\prs{j}$. Assumption \ref{grouping assumption: ch} ensures $T$ is indeed a column-stochastic matrix. Thus, the probability measure $\Bar{P}_T$ on $\cX\times\cY\times\cY$ satisfies $\alpha\prs{i} = \Bar{P}_T(\Tilde{Y} = i)$ and $P_{\gamma_i}\prs{\cdot}=\Bar{P}_T(\cdot\mid\Tilde{Y} = i)$, which further implies $\gamma_i\prs{c}=\Bar{P}_T(Y=c\mid\Tilde{Y} = i)$. We confirm these facts in Section \ref{sec:cfm prob model} in the appendix. Such construction transforms LLP with $C$ bags into LLN with an estimable noise transition matrix $T$. Each element of a bag can then be viewed as a triplet $(X, Y, \tilde{Y})$, with $Y$ unobserved, such that $\prs{X, Y}$ is drawn from $P_{\gamma_{\tilde{Y}}}$. After assigning the noisy labels, we have a dataset $\bigcup_{c=1}^C \cbs{\prs{X^c_{j}, c}: j \in \NN_{\abs{b_c}}}$ along with the noise transition matrix $T$. This allows us to leverage the forward correction loss $\ell_{T}$ to minimize the objective $\cR_{\ell_{T}, P_T} (f) = \EE_{P_T} [\ell_T(f(X), \Tilde{Y})]$ which can be estimated by the empirical risk 
$\sum_{c=1}^C \frac{\alpha(c)}{\abs{b_c}} \sum_{j=1}^{\abs{b_c}} \ell_{T} \prs{f(X_j^c), c}$.






\subsection{The General Case: Reduction to LMNTM}
\label{sec:generalcase}


More generally, consider LLP with $NC$ bags, $N\in \NN$. We propose to randomly partition the bags into $N$ groups, each with $C$ bags indexed from $1$ to $C$. Let $k_{i,c}$ denote the index of the $c$-th bag in the $i$-th group. Thus, $b_{k_{i,c}}$ is the $c$-th bag in the $i$-th group and it is governed by $\gamma_{k_{i,c}}$. For $i\in\NN_N$, define the matrix $\Gamma_i \in \RR^{C\times C}$ by $\Gamma_i(c_1, c_2) = \gamma_{k_{i, c_1}}(c_2)$,  $\forall c_1, c_2 \in \cY$. We make the following assumption on the $\Gamma_i$'s and $\sigma$: 
\begin{assumption} \label{grouping assumption all groups}
For each $i\in \NN_N$, $\exists \text{ unique } \alpha_i \in \mathring{\Delta}^C$ $s.t.$  $\Gamma_i^{tr} \alpha_i = \sigma$.
\end{assumption}
Thus, every group $i$ can be modeled as above as an LLN problem with noise transition matrix $T_i$ where $T_i\prs{c_1, c_2} = \gamma_{k_{i, c_1}}\prs{c_2} \alpha_i(c_1)/\sigma\prs{c_2}$. Data points in the bag assigned with noisy label $c$ in the $i$-th group can be viewed as drawn $i.i.d.$ from the class conditional distribution $P_{T_i}\prs{\cdot\mid c}$. This problem now maps directly to LMNTM as described in Section \ref{subset: lmntm_setup}, and satisfies the associated performance guarantees. In the next subsection, we spell out the associated algorithm. 




\subsection{Algorithms} \label{subsec: llpfc algo}
\begin{algorithm}[tb] \caption{LLPFC-ideal} \label{algo:LLPFC}
\begin{multicols}{2}
\begin{algorithmic}[1]
 \STATE {\bfseries Input:} {$\cbs{\prs{b_k, \gamma_k}}_{k=1}^{NC}$ and $w \in \Delta^N$ where  $b_k = \cbs{X^k_j: j \in \NN_{\abs{b_k}}}$}.
 \STATE Randomly partition the bags into $N$ groups $\cbs{G_i}_{i=1}^N$ $s.t.$ $G_i = \cbs{(b_{k_{i,c}},\gamma_{k_{i,c}} ): c\in \cY}$ 
 \FOR {$i=1:N$} {
     \STATE $\Gamma_i \leftarrow [\gamma_{k_{i, 1}}, \gamma_{k_{i, 2}}, \dots, \gamma_{k_{i, C}}]^{tr}$
     \STATE $\alpha_i \leftarrow \Gamma_i^{-tr} \sigma$
     \FOR {$c_1=1:C, c_2=1:C$} {
     \STATE $T_i(c_1, c_2) \leftarrow \gamma_{k_{i,c_1}} (c_2) \alpha_i(c_1)/\sigma(c_2)$
     }
     \ENDFOR
 } 
 \ENDFOR 
 \STATE Train $f$ with the empirical objective \eqref{llpfc-ideal: empirical risk}
\end{algorithmic}
\end{multicols}
\end{algorithm}
\begin{algorithm}[tb] \caption{LLPFC-uniform} \label{algo:LLPFC-uniform}
\begin{multicols}{2}
\begin{algorithmic}[1]
 \STATE {\bfseries Input:} {$\cbs{\prs{b_k, \hat{\gamma}_k}}_{k=1}^{NC}$ and $w \in \Delta^N$ where  $b_k = \cbs{X^k_j: j \in \NN_{\abs{b_k}}}$}.
 \STATE Partition the bags as step 2 in Algorithm \ref{algo:LLPFC}.
 \FOR {$i=1:N$} {
     \STATE $\hat{\Gamma}_i \leftarrow [\hat{\gamma}_{k_{i, 1}}, \hat{\gamma}_{k_{i, 2}}, \dots, \hat{\gamma}_{k_{i, C}}]^{tr}$
     \STATE $n_i \leftarrow \sum_{c=1}^C \abs{b_{k_{i, c}}}$
        {\STATE $\hat{\alpha}_i(c) \leftarrow |b_{k_{i,c}}|/n_i$ for each $c=1:C$}
     \STATE $\hat{\sigma}_i \leftarrow \hat{\Gamma}_i^{tr} \hat{\alpha}_i$
     \FOR {$c_1=1:C,c_2=1:C$} {
     \STATE $\hat{T}_i(c_1, c_2) \leftarrow \hat{\gamma}_{k_{i,c_1}} (c_2) \hat{\alpha}_i(c_1)/\hat{\sigma}_i(c_2)$
     }
     \ENDFOR
 } 
 \ENDFOR 
 \STATE Train with 
 $\sum_{i, c} \frac{w_i}{n_i} \sum_{j} {\ell_{\hat{T}_i}} (f(X^{k_{i,c}}_j), c)$.
\end{algorithmic}
\end{multicols}
\end{algorithm}
\begin{algorithm}[tb] \caption{LLPFC-approx} \label{algo:LLPFC-approx}
\begin{multicols}{2}
\begin{algorithmic}[1]
 \STATE {\bfseries Input:} {$\cbs{\prs{b_k, \hat{\gamma}_k}}_{k=1}^{NC}$ and $w \in \Delta^N$ where  $b_k = \cbs{X^k_j: j \in \NN_{\abs{b_k}}}$}.
  \STATE $\hat{\sigma} \leftarrow \frac{\sum_{k=1}^{NC} \abs{b_k}\hat{\gamma}_k}{\sum_{k=1}^{NC}{\abs{b_k}}}$
  \STATE Partition the bags as step 2 in Algorithm \ref{algo:LLPFC}.
 \FOR {$i=1:N$} {
     \STATE $\hat{\Gamma}_i \leftarrow [\hat{\gamma}_{k_{i, 1}}, \hat{\gamma}_{k_{i, 2}}, \dots, \hat{\gamma}_{k_{i, C}}]^{tr}$
     \STATE $\hat{\alpha}_i \leftarrow \argmin_{{\alpha} \in \Delta^C} ||\hat{\sigma} -\hat{\Gamma}_i^{tr} {\alpha}||_2^2$
     \STATE $\hat{\sigma}_i \leftarrow \hat{\Gamma}_i^{tr} \hat{\alpha}_i$
     \FOR {$c_1=1:C,c_2=1:C$} {
     \STATE $\hat{T}_i(c_1, c_2) \leftarrow \hat{\gamma}_{k_{i,c_1}} (c_2) \hat{\alpha}_i(c_1)/\hat{\sigma}_i(c_2)$
     }
     \ENDFOR
 } 
 \ENDFOR 
 \STATE Train with
 $\sum_{i, c} \frac{w_i \hat{\alpha}_i (c)}{\abs{b_{k_{i, c}}}} \sum_{j} {\ell_{\hat{T}_i}}  (f(X^{k_{i,c}}_j), c)$
\end{algorithmic}
\end{multicols}
\end{algorithm}
As above, assume we have $NC$ bags where $N \in \NN$. Let each bag $b_k$ be governed by $\gamma_k \in \Delta^C$ and be annotated by label proportion $\hat{\gamma}_k$. We first present the LLPFC-ideal algorithm in an ideal setting where $\sigma$, the $\gamma_k$'s and the $\alpha_i$'s are known precisely and Assumption \ref{grouping assumption all groups} holds. We then present the real-world adaptations LLPFC-uniform and LLPFC-approx in practical settings. 


The LLPFC-ideal algorithm is presented in Algorithm \ref{algo:LLPFC}.  We follow the idea in section \ref{sec:generalcase} to partition the bags into $N$ groups of $C$ bags, and model each group as an LLN problem. In Algorithm \ref{algo:LLPFC}, we assume $\gamma_k$ and $\sigma$ are known and Assumption \ref{grouping assumption all groups} holds. The theoretical analysis in Section \ref{sec:mntm} is immediately applicable to the LLPFC-ideal algorithm. We partition the bags by uniformly randomly partitioning the set of indices $\NN_{NC}$ into disjoint subsets $\cbs{k_{i,c}: c \in \cY}$, $i \in \NN_{N}$, where $k_{i,c}$ denotes the index of the $c$-th bag in the $i$-th group. We denote the inverse transpose of $\Gamma_i$ by ${\Gamma_i}^{-tr}$.


In practice, when $\gamma_k$ is unknown, we replace $\gamma_k$ with $\hat{\gamma}_k$ as a plug-in method. Hence, we work with $\hat{\sigma} = \frac{\sum_{k=1}^{NC} \abs{b_k}\hat{\gamma}_k}{\sum_{k=1}^{NC}{\abs{b_k}}}$ and $\hat{\Gamma}_i = [\hat{\gamma}_{k_{i, 1}}, \hat{\gamma}_{k_{i, 2}}, \dots, \hat{\gamma}_{k_{i, C}}]^{tr}$ instead of $\sigma$ and $\Gamma_i$ in Algorithm \ref{algo:LLPFC}, respectively. Here $\hat{\sigma}$ is the label proportion of all training data points and we use it as an estimate of the clean prior $\sigma$.  Likewise, $\alpha_i = \Gamma_i^{-tr} \sigma$ in Algorithm \ref{algo:LLPFC} should be replaced with $\hat{\alpha}_i = \hat{\Gamma}_i^{-tr} \hat{\sigma}$ and we would like to use $\hat{\Gamma}_i$, $\hat{\sigma}$, and $\hat{\alpha}_i$ to calculate $\hat{T}_i$ as an estimate of $T_i$. For this to make sense, we need $\hat{\alpha}_i = \hat{\Gamma}_i^{-tr} \hat{\sigma} \in \mathring{\Delta}^C$, which is equivalent to $\hat{\sigma}$ being in the interior of the convex hull of $\cbs{\hat{\gamma}_{k_{i, c}}:c \in\cY}$ for all $i$. However, this may not be the case in practice. Thus, we consider two heuristics to estimate $\hat{T}_i$ as real-world adaptations of the LLPFC-ideal algorithm. The first, called LLPFC-uniform, is presented in Algorithm \ref{algo:LLPFC-uniform} which sets $\hat\alpha_i$ by counting the occurrences of the noisy labels. This is motivated by our model wherein $\alpha_i$ is the noisy class prior for the $i$-th group. The second, called LLPFC-approx, is presented in Algorithm \ref{algo:LLPFC-approx} and sets $\hat{\alpha}_i$ to be the solution of $\argmin_{\alpha \in \Delta^C} ||\hat{\sigma} -\hat{\Gamma}_i \alpha||_2^2$. It should be noted that in both practical algorithms, we use a different $\hat{\sigma}_i$ as an estimate of $\sigma$ for each group, to ensure that each $\hat{T}_i$ is a column-stochastic matrix. 
In experiments where we have $NC+k$ number of bags with $0< k <C$, we can randomly resample $NC$ number of bags and regroup them in every few epochs. Both real-world adaptations perform reasonably well in experiments.



\section{Experiments} \label{sec:exp} \footnote{Code is available at \url{https://github.com/Z-Jianxin/LLPFC}}
\begin{table*}[!htbp]
\caption{Test Accuracy for Wide ResNet-16-4 \label{tab: wideresnet.acc}}
\resizebox{1.0\textwidth}{!}{
  \begin{tabular}{c c c c c c c c c}
    \hline 
    Data set & Method & 32 & 64 & 128 & 256 & 512 & 1024 & 2048  \\
    \hline
    \multirow{4}{*}{CIFAR10} 
    & KL & 
    .4255 $\pm$ .13 &
    .6817 $\pm$ .16 &
    .5346 $\pm$ .11 &
    .3749 $\pm$ .14 &
    .2938 $\pm$ .04 &
    Out of RAM &
    Out of RAM
    \\
    
    & LLPVAT & 
     .4911 $\pm$ .15 &
     .5137 $\pm$ .22 &
     .4744 $\pm$ .12 &
     .4423 $\pm$ .16 &
     Out of RAM  &
     Out of RAM  &
     Out of RAM
    \\
    
    & LLPFC-uniform & 
    .7926 $\pm$ .01 &
    \textbf{.7683 $\pm$ .02} &
    .7399 $\pm$ .02 &
    .7381 $\pm$ .01 &
    .7224 $\pm$ .01 &
    .7182 $\pm$ .01 &
    .6925 $\pm$ .03
    \\
    
    & LLPFC-approx & 
    \textbf{.7993 $\pm$ .00} &
    .7671 $\pm$ .01 &
    \textbf{.7528 $\pm$ .01}&
    \textbf{.7404 $\pm$ .00}&
    \textbf{.7409 $\pm$ .02}&
    \textbf{.7205 $\pm$ .03}&
    \textbf{.7283 $\pm$ .02}
    \\
    
    \hline
    \multirow{4}{*}{SVHN} 
    & KL & 
    .2465 $\pm$ .10 &
    .1152 $\pm$ .07 &
    .1022 $\pm$ .03 &
    .1294 $\pm$ .04 &
    .1039 $\pm$ .04 &
    Out of RAM &
    Out of RAM 
    \\
    & LLPVAT & 
    .2675 $\pm$ .36 &
    .1398 $\pm$ .08 &
    .1004 $\pm$ .03 &
    .1294 $\pm$ .04 &
    Out of RAM &
    Out of RAM &
    Out of RAM 
    \\
    & LLPFC-uniform& 
    \textbf{.9012 $\pm$ .02} &
    \textbf{.8855 $\pm$ .02} &
    .8760 $\pm$ .02 &
    .8736 $\pm$ .01 &
    .8681 $\pm$ .02 &
    \textbf{.8709 $\pm$ .02} &
    .8717 $\pm$ .01 
    \\
    & LLPFC-approx & 
    .8903 $\pm$ .02 &
    .8844 $\pm$ .02 &
    \textbf{.8815 $\pm$ .03} &
    \textbf{.8808 $\pm$ .01} &
    \textbf{.8771 $\pm$ .02} &
    .8701 $\pm$ .02 &
    \textbf{.8738 $\pm$ .01}
    \\
    
    \hline
    \multirow{4}{*}{EMNIST} 
    & KL & 
    .8413 $\pm$ .04 &
    .8637 $\pm$ .04 &
    .9111 $\pm$ .00 &
    .5361 $\pm$ .12 &
    .0845 $\pm$ .01 &
    .0826 $\pm$ .01 &
    Out of RAM
    \\
    & LLPVAT & 
    .8254 $\pm$ .04 &
    .9045 $\pm$ .02 &
    \textbf{.9136 $\pm$ .00} &
    .5071 $\pm$ .09 &
    .0859 $\pm$ .01 &
    Out of RAM &
    Out of RAM 
    \\
    & LLPFC-uniform& 
    \textbf{.9165 $\pm$ .01} &
    .9061 $\pm$ .01 &
    .9015 $\pm$ .01 &
    .8790 $\pm$ .03 &
    .8886 $\pm$ .02 &
    .8461 $\pm$ .05 &
    .8817 $\pm$ .01 
    \\
    & LLPFC-approx & 
    .9092 $\pm$ .01 &
    \textbf{.9074 $\pm$ .01} &
    .9065 $\pm$ .00 &
    \textbf{.8993 $\pm$ .01} &
    \textbf{.9048 $\pm$ .00} &
    \textbf{.8969 $\pm$ .01} &
    \textbf{.9007 $\pm$ .01}
    \\
  \end{tabular}
}
\caption{Test Accuracy for ResNet18 \label{tab: resnet.acc}}
\resizebox{1.0\textwidth}{!}{
  \begin{tabular}{c c c c c c c c c}
    \hline 
    Data set & Method & 32 & 64 & 128 & 256 & 512 & 1024 & 2048  \\
    \hline
    \multirow{4}{*}{CIFAR10} & KL & 
     .7837 $\pm$  .01&
     \textbf{.7565 $\pm$  .01}&
     .6918 $\pm$  .01&
     .6106 $\pm$  .04&
     .5696 $\pm$  .05&
     .5197 $\pm$  .04&
     .4576 $\pm$  .03
    \\
    & LLPVAT & 
    \textbf{.7907 $\pm$ .01} &
    .7499 $\pm$ .01 &
    \textbf{.6946 $\pm$ .01} &
    \textbf{.6115 $\pm$ .03} &
    .5670 $\pm$ .04 &
    .4881 $\pm$ .02 &
    .4624 $\pm$ .02
    \\
    & LLPFC-uniform& 
    .6601 $\pm$ .01 &
    .6310 $\pm$ .01 &
    .5867 $\pm$ .01 &
    .5603 $\pm$ .01 &
    .5670 $\pm$ .01 &
    .5623 $\pm$ .01 &
    .5288 $\pm$ .03
    \\
    & LLPFC-approx & 
    .6567 $\pm$ .01 &
    .6136 $\pm$ .01 &
    .5997 $\pm$ .01 &
    .5931 $\pm$ .02 &
    \textbf{.6062 $\pm$ .01} &
    \textbf{.6169 $\pm$ .01} &
    \textbf{.5591 $\pm$ .04}
    \\
    
    \hline
    \multirow{4}{*}{SVHN} 
    & KL & 
    .1279 $\pm$ .06 &
    .0716 $\pm$ .01 &
    .3042 $\pm$ .30 &
    .1026 $\pm$ .04 &
    .2489 $\pm$ .24 &
    .3123 $\pm$ .29 &
    .2797 $\pm$ .12 
    \\
    & LLPVAT & 
    .1279 $\pm$ .06 &
    .1939 $\pm$ .27 &
    .6962 $\pm$ .34 &
    .3673 $\pm$ .39 &
    .4003 $\pm$ .31 &
    .3999 $\pm$ .33 &
    .3736 $\pm$ .24
    \\
    & LLPFC-uniform& 
    .8823 $\pm$ .01 &
    .8644 $\pm$ .01 &
    .8433 $\pm$ .01 &
    .8390 $\pm$ .01 &
    .8360 $\pm$ .00 &
    .8086 $\pm$ .02 &
    .8188 $\pm$ .01 
    \\
    & LLPFC-approx & 
    \textbf{.8824 $\pm$ .01} &
    \textbf{.8672 $\pm$ .01} &
    \textbf{.8570 $\pm$ .01} &
    \textbf{.8483 $\pm$ .01} &
    \textbf{.8492 $\pm$ .01} &
    \textbf{.8498 $\pm$ .01} &
    \textbf{.8534 $\pm$ .01}
    \\
    
    \hline
    \multirow{4}{*}{EMNIST} & KL & 
    \textbf{.9319 $\pm$ .00} &
    .9295 $\pm$ .00 &
    \textbf{.9306 $\pm$ .00} &
    .9269 $\pm$ .00 &
    \textbf{.9267 $\pm$ .00} &
    \textbf{.9239 $\pm$ .00} &
    .9106 $\pm$ .01 
    \\
    & LLPVAT & 
    .9308 $\pm$ .00 &
    \textbf{.9299 $\pm$ .00} &
    .9299 $\pm$ .00 &
    \textbf{.9281 $\pm$ .00} &
    .9248 $\pm$ .00 &
    .9222 $\pm$ .00 &
    \textbf{.9128 $\pm$ .00}
    \\
    & LLPFC-uniform & 
    .9144 $\pm$ .00  &
    .8954 $\pm$ .00 &
    .8744 $\pm$ .00 &
    .8600 $\pm$ .00 &
    .8448 $\pm$ .00 &
    .8388 $\pm$ .01 &
    .8245 $\pm$ .01 
    \\
    & LLPFC-approx & 
    .9146 $\pm$ .00 &
    .8998 $\pm$ .00 &
    .8874 $\pm$ .00 &
    .8764 $\pm$ .01 &
    .8670 $\pm$ .00 &
    .8736 $\pm$ .01 &
    .8660 $\pm$ .01
    \\
  \end{tabular}
}
\caption{Test Accuracy for VGG16 \label{tab: vgg16.acc}}
\resizebox{1.0\textwidth}{!}{
  \begin{tabular}{c c c c c c c c c}
    \hline 
    Data set & Method & 32 & 64 & 128 & 256 & 512 & 1024 & 2048  \\
    \hline
    \multirow{4}{*}{CIFAR10} & KL & 
     .2513 $\pm$ .11 &
     .2130 $\pm$ .06 &
     .1794 $\pm$ .04 &
     .1160 $\pm$ .02 &
     .1117 $\pm$ .01 &
     .1221 $\pm$ .00 &
     .1049 $\pm$ .01
    \\
    & LLPVAT& 
    .4634 $\pm$ .07 &
    .2093 $\pm$ .03 &
    .1399 $\pm$ .03 &
    .1145 $\pm$ .02 &
    .1172 $\pm$ .02 &
    .1189 $\pm$ .00 &
     Out of RAM 
    \\
    & LLPFC-uniform & 
    \textbf{.7602 $\pm$ .00} &
    \textbf{.7372 $\pm$ .01} &
    \textbf{.7300 $\pm$ .01} &
    \textbf{.7226 $\pm$ .01} &
    \textbf{.7136 $\pm$ .01} &
    \textbf{.7111 $\pm$ .01} &
    \textbf{.7033 $\pm$ .03}
    \\
    & LLPFC-approx & 
    .7566 $\pm$ .00 &
    .7310 $\pm$ .01 &
    .7003 $\pm$ .01 &
    .7004 $\pm$ .01 &
    .6870 $\pm$ .03 &
    .6857 $\pm$ .02 &
    .6645 $\pm$ .03
    \\
    
    \hline
    \multirow{4}{*}{SVHN} 
    & KL & 
    .1277 $\pm$ .06 &
    .0893 $\pm$ .04 &
    .1054 $\pm$ .05 &
    .1024 $\pm$ .05 &
    .1104 $\pm$ .04 &
    .0885 $\pm$ .01 &
    .1372 $\pm$ .03
    \\
    & LLPVAT& 
    .1117 $\pm$ .05 &
    .0736 $\pm$ .01 &
    .1051 $\pm$ .05 &
    .1023 $\pm$ .06 &
    .1125 $\pm$ .04 &
    .1061 $\pm$ .05 &
    Out of RAM
    \\
    & LLPFC-uniform & 
    .4177 $\pm$ .15 &
    .4708 $\pm$ .23 &
    \textbf{.5402 $\pm$ .21} &
    \textbf{.1734 $\pm$ .11} &
    \textbf{.4249 $\pm$ .30} &
    \textbf{.5691 $\pm$ .27} &
    \textbf{.6869 $\pm$ .13}
    \\
    & LLPFC-approx & 
    \textbf{.4299 $\pm$ .28} &
    \textbf{.4994 $\pm$ .23} &
    .1091 $\pm$ .04 &
    .1188 $\pm$ .05 &
    .1903 $\pm$ .14 &
    .4097 $\pm$ .17 &
    .4429 $\pm$ .18
    \\
    
    \hline
    \multirow{4}{*}{EMNIST} & KL & 
    .5952 $\pm$.45 &
    .2348 $\pm$.22 &
    .0974 $\pm$.01 &
    .0842 $\pm$.02 &
    .0702 $\pm$.01 &
    .0692 $\pm$.01 &
    .0597 $\pm$.02
    \\
    & LLPVAT& 
    .8593 $\pm$ .16 &
    .3329 $\pm$ .33 &
    .1042 $\pm$ .01 &
    .0833 $\pm$ .02 &
    .0696 $\pm$ .01 &
    .0711 $\pm$ .00 &
     Out of RAM 
    \\
    & LLPFC-uniform& 
    \textbf{.9311 $\pm$ .00}  &
    .9279 $\pm$ .00  &
    \textbf{.9258 $\pm$ .00}  &
    \textbf{.9242 $\pm$ .00}  &
    \textbf{.9239 $\pm$ .00}  &
    \textbf{.9233 $\pm$ .00}  &
    \textbf{.9220 $\pm$ .00}  
    \\
    & LLPFC-approx & 
    .9310 $\pm$ .00 &
    \textbf{.9280 $\pm$ .00} &
    .9249 $\pm$ .00 &
    .9240 $\pm$ .00 &
    .9227 $\pm$ .00 &
    .9206 $\pm$ .00 &
    .9205 $\pm$ .00
    \\
  \end{tabular}
}
\caption{Test Accuracy for LLPGAN architecture \label{tab: llpgan.acc}}
\resizebox{1.0\textwidth}{!}{
  \begin{tabular}{c c c c c c c c c}
    \hline 
    Data set & Method & 32 & 64 & 128 & 256 & 512 & 1024 & 2048  \\
    \hline
    \multirow{3}{*}{CIFAR10} 
    & LLPGAN& 
    .3630 $\pm$ .01 &
    .3133 $\pm$ .02 &
    .3328 $\pm$ .03 &
    .3363 $\pm$ .03 &
    .3460 $\pm$ .03 &
    .2824 $\pm$ .05 &
    .2236 $\pm$ .08
    \\
    & LLPFC-uniform & 
    .6145 $\pm$ .01 &
    .5826 $\pm$ .01 &
    .5565 $\pm$ .03 &
    .5452 $\pm$ .01 &
    .5511 $\pm$ .02 &
    .5358 $\pm$ .01 &
    .5438 $\pm$ .03
    \\
    & LLPFC-approx & 
    \textbf{.6169 $\pm$ .01} &
    \textbf{.5875 $\pm$ .01} &
    \textbf{.5642 $\pm$ .02}&
    \textbf{.5687 $\pm$ .02} &
    \textbf{.5621 $\pm$ .03} &
    \textbf{.5610 $\pm$ .01} &
    \textbf{.5567 $\pm$ .02}
    \\
    
    \hline
    \multirow{3}{*}{SVHN}
    & LLPGAN& 
    .2378 $\pm$ .24 &
    .7135 $\pm$ .06 &
    .7680 $\pm$ .04 &
    .6058 $\pm$ .29 &
    .4863 $\pm$ .22 &
    .1725 $\pm$ .06 &
    .1382 $\pm$ .04
    \\
    & LLPFC-uniform & 
    \textbf{.8800 $\pm$ .00} &
    \textbf{.8581 $\pm$ .01} &
    \textbf{.8480 $\pm$ .01} &
    .8393 $\pm$ .01 &
    .8347 $\pm$ .01 &
    .8258 $\pm$ .01 &
    .8327 $\pm$ .01
    \\
    & LLPFC-approx & 
    .8779 $\pm$ .01 &
    .8519 $\pm$ .01 &
    .7061 $\pm$ .33 &
    \textbf{.8453 $\pm$ .02} &
    \textbf{.8423 $\pm$ .01} &
    \textbf{.8386 $\pm$ .01} &
    \textbf{.8527 $\pm$ .01}
    \\
  \end{tabular}
}
\end{table*}
We compare against three previous works that have studied LLP applying deep learning to image data: \textcite{DulacArnold2019DeepML} study the KL loss described in the introduction, and a novel loss based on optimal transport. They find that KL performs just as well as the novel loss. \textcite{Liu19LLPGAN} employ the KL loss within a generative adversarial framework (LLPGAN).  \textcite{Tsai20LLPVAT} propose augmenting the KL loss with a regularizer from semi-supervised learning and show improved performance (LLPVAT). We compare both LLPFC-uniform and LLPFC-approx against the KL loss, LLPGAN, and LLPVAT to clearly establish which empirical objective is better. Recent papers on multiclass LLP for which code is not available were not included \cite{Liu21LLPGANPLOT, kobayashi22arxiv}.

We generate bags with fixed, equal sizes in $\cbs{32, 64, 128, 256, 512, 1024, 2048}$. To generate each bag, we first sample a label proportion $\gamma$ from the uniform distribution on $\Delta^C$. Then we sample data points from a benchmark dataset without replacement using a multinomial distribution with parameter $\gamma$. It should be noted that \textcite{Tsai20LLPVAT}, \textcite{DulacArnold2019DeepML}, and \textcite{Liu19LLPGAN} generate bags by shuffling all data points and making every $B$ data points a bag where $B$ is a fixed bag size. Their method is equivalent to sampling data points without replacement using a multinomial distribution with a fixed parameter $\gamma=\cbs{\frac{1}{C}, \frac{1}{C}, \dots, \frac{1}{C}}$. As noted by \textcite{binary}, this leads to bags with very similar label proportions which makes the learning task much more challenging. 

We repeat each experiment 5 times and report the mean test accuracy and standard deviation. All models are trained on a single Nvidia Tesla v100 GPU with 16GB RAM. In our implementation of LLPFC algorithms, the weight $w$ is set to be $(\frac{1}{N}, \dots, \frac{1}{N}) \in \Delta^N$ and our choice of the proper composite loss is the cross-entropy loss. 

For the comparison against KL and LLPVAT, we perform experiments on three benchmark image datasets: the ``letter'' split of EMNIST \cite{cohen2017emnist}, SVHN \cite{Netzer2011SVHN}, and CIFAR10 \cite{Kriz2012cifar}. To show that our approach is robust to the choice of architecture, we experiment with three different networks: Wide ResNet-16-4 \cite{WideResNet}, ResNet18 \cite{ResNet}, and  VGG16 \cite{Simonyan15VGG}. We train these networks with the parameters suggested in the original papers.  The test accuracies are reported in Tables \ref{tab: wideresnet.acc}, \ref{tab: resnet.acc}, and \ref{tab: vgg16.acc}. Since convergence in the GAN framework is sensitive to the choice of architecture and hyperparameters, we compare LLPFC against LLPGAN using the architecture proposed in the original paper along with the hyperparameters suggested in their code\footnote{https://github.com/liujiabin008/LLP-GAN}. It should be noted that for LLPFC we only use the discriminator for classification and did not use the generator to augment data.  Since  \textcite{Liu19LLPGAN} only provide hyperparameters for colored images, we perform experiments on SVHN and CIFAR10 only. The test accuracies are reported in Table \ref{tab: llpgan.acc}.

LLPFC-uniform and LLPFC-approx substantially outperform the competitors in a clear majority of settings. 
The experiment results clearly establish our methods as the state-of-the-art by a substantial margin. 
All three competitors perform gradient descent with minibatches of bags and the GPU at times runs out of memory when the bag size is large. Our implementation, which also uses stochastic optimization, does not suffer from this phenomenon. Full experimental details are in the appendix.

\section{Conclusions and Future Work} \label{sec:conclusion}
We propose a theoretically supported approach to LLP by reducing it to learning with label noise and using the forward correction (FC) loss. An excess risk bound and generalization error analysis are established. Our approach outperforms leading existing methods in deep learning scenarios across multiple datasets and architectures. A limitation of our approach is that the theory makes an assumption that may not be verifiable in practice. Future research directions include optimizing the grouping of bags and adapting LLPFC to other objectives beyond accuracy.

\textbf{Acknowledgement}
The authors were supported in part by the National Science Foundation under awards 1838179 and 2008074, and by the Department of Defense, Defense Threat Reduction Agency under award HDTRA1-20-2-0002.

\clearpage

\printbibliography

\section*{Checklist}

\begin{enumerate}
\item For all authors...
\begin{enumerate}
  \item Do the main claims made in the abstract and introduction accurately reflect the paper's contributions and scope?
    \answerYes{}
  \item Did you describe the limitations of your work?
    \answerYes{Limitations are described in Section \ref{sec:conclusion}.}
  \item Did you discuss any potential negative societal impacts of your work?
    \answerNA{}{}
  \item Have you read the ethics review guidelines and ensured that your paper conforms to them?
    \answerYes{}
\end{enumerate}

\item If you are including theoretical results...
\begin{enumerate}
  \item Did you state the full set of assumptions of all theoretical results?
    \answerYes{}
        \item Did you include complete proofs of all theoretical results?
    \answerYes{All proofs are included in the appendix.}
\end{enumerate}

\item If you ran experiments...
\begin{enumerate}
  \item Did you include the code, data, and instructions needed to reproduce the main experimental results (either in the supplemental material or as a URL)?
    \answerYes{Code is included in the supplemental material. Datasets are public and we provide code to download them. We include a README file with instructions on how to reproduce experimental results.}
  \item Did you specify all the training details (e.g., data splits, hyperparameters, how they were chosen)?
    \answerYes{A concise description of experiments is in Section \ref{sec:exp} in the main paper with full details in section \ref{suppsec: exp} in the appendix.}
        \item Did you report error bars (e.g., with respect to the random seed after running experiments multiple times)?
    \answerYes{In table \ref{tab: wideresnet.acc}, \ref{tab: resnet.acc}, \ref{tab: vgg16.acc} and \ref{tab: llpgan.acc}, we report the mean and standard deviation for 5 trials with different random seeds.}
        \item Did you include the total amount of compute and the type of resources used (e.g., type of GPUs, internal cluster, or cloud provider)?
    \answerNo{We provide the hardware information in Section \ref{sec:exp} but not the amount of compute. The computational time varies for different experimental settings. It would be too exhaustive to present them given that we have nearly 300 different settings.}
\end{enumerate}

\item If you are using existing assets (e.g., code, data, models) or curating/releasing new assets...
\begin{enumerate}
  \item If your work uses existing assets, did you cite the creators?
    \answerYes{Yes. Creators of code, data, and models are all cited.}
  \item Did you mention the license of the assets?
    \answerNA{We did not directly run experiments using others' code. However, we implement their algorithms with their code as a reference.}
  \item Did you include any new assets either in the supplemental material or as a URL?
    \answerYes{Our experiment code is included in the supplemental materials.}
  \item Did you discuss whether and how consent was obtained from people whose data you're using/curating?
    \answerNA{}
  \item Did you discuss whether the data you are using/curating contains personally identifiable information or offensive content?
    \answerNA{}
\end{enumerate}

\item If you used crowdsourcing or conducted research with human subjects...
\begin{enumerate}
  \item Did you include the full text of instructions given to participants and screenshots, if applicable?
    \answerNA{}
  \item Did you describe any potential participant risks, with links to Institutional Review Board (IRB) approvals, if applicable?
    \answerNA{}
  \item Did you include the estimated hourly wage paid to participants and the total amount spent on participant compensation?
    \answerNA{}
\end{enumerate}

\end{enumerate}


\appendix

\newpage
\appendix
\onecolumn


\section{Experiment Details}
\label{suppsec: exp}

\subsection{Datasets}
We perform experiments and compare against the KL loss of \textcite{DulacArnold2019DeepML} and LLPVAT of \textcite{Tsai20LLPVAT} on three benchmark datasets of image classification: the ``letter" split of EMNIST \cite{cohen2017emnist}, SVHN \cite{Netzer2011SVHN}, and CIFAR10 \cite{Kriz2012cifar}. We also compare our methods against LLPGAN of \textcite{Liu19LLPGAN}  on SVHN and CIFAR10. To generate each bag, we first sample a label proportion $\gamma$ from the uniform distribution on $\Delta^C$ and then sample data points without replacement using a multinomial distribution with parameter $\gamma$. The generated bags have fixed and equal sizes in $\cbs{32, 64, 128, 256, 512, 1024, 2048}$. For SVHN and CIFAR10, $32 \times 1280 = 40960$ data points are sampled for every bag size. For EMNIST, the number of sampled data points is $32 \times 3328 = 106496$. 

\subsection{Architecture and Hyperparameters}
To compare our methods against the KL loss and LLPVAT, we train Wide ResNet-16-4 \cite{WideResNet}, ResNet18 \cite{ResNet}, and VGG16 \cite{Simonyan15VGG} with the hyperparameters suggested in the original papers. For the comparison against LLPGAN, we use the discriminator architecture proposed in \textcite{Liu19LLPGAN} and the hyperparameters suggested in their code\footnote{\label{LLPGAN_code}https://github.com/liujiabin008/LLP-GAN}.
It should be noted that KL, LLPVAT, and LLPGAN are all required to backpropagate on minibatches of bags and our method does not have such constraint. For all methods, to avoid overfitting, we apply a standard data augmentation procedure: 4 pixels with value 0 are padded on each side, and a crop of the original size is randomly sampled from the padded image or its horizontal flip. 

\subsubsection{Wide ResNet-16-4}
For all datasets, we use SGD with Nesterov momentum with weight decay set to 0.0005, dampening to 0, and momentum to 0.9. The minibatch size is set to 128 for our method. On CIFAR, the initial learning rate is set to 0.01, which is divided by 5 at 60, 120 and 160 epochs, and the network is trained for total 200 epochs. On SVHN and EMNIST, the initial learning rate is set to 0.01, which is divided by 10 at 80 and 120 epochs, and the network is trained for 160 epochs. The dropout probability is 0.3 for CIFAR and 0.4 for both SVHN and EMNIST.

\subsubsection{ResNet18}
We use SGD and weight decay is set to 0.0001 and momentum to 0.9. The minibatch size is set to 128 for our method. The model is trained for 500 epochs for all datasets. The learning rate is initialized to be 0.1 and divided by 10 at 250 and 375 epochs.

\subsubsection{VGG16}
We use SGD and weight decay is set to 0.0005 and momentum to 0.9. The minibatch size is set to 256 for our method. Dropout ratio is set to 0.5 for the first two fully-connected layers. The learning rate was initially set to 0.01. We train the model for 74 epochs in total. In the original paper of VGG16 \cite{Simonyan15VGG}, the learning rate is decreased when validation accuracy stops improving and it is decreased 3 times in total. In our experiment, while we do not assume access to fully labeled validation dataset, we divide the learning rate by 10 at 19, 37, and 56 epochs.

\subsubsection{LLPGAN's discriminator}
The neural network is trained for 3000 epochs and optimized by Adam \cite{Adam} with a learning rate 0.0003. The minibatch size is set to 128 for our method. The $\beta_1$ and $\beta_2$ parameters for Adam are set to be 0.5 and 0.999, respectively. 

\subsection{KL Loss}
\label{subsec:KL}
Recall $C\in \NN$ denotes the number of classes and $\cX$ denotes the feature space. Let $f$ be a function that maps $\cX$ to $\Delta^C$, and $M\in \NN$ be the total number of bags. Let $\cbs{(b_i, \hat{\gamma}_i)}_{i=1}^M$ be the 
bags and empirical label proportions where $b_i = (X^i_{1}, X^i_{2}, \dots, X^i_{m_i})$ and $m_i$ is the size of bag $b_i$. The KL loss of \textcite{DulacArnold2019DeepML} seeks to minimize the empirical objective $\Bar{\cL}_{prop} = -\frac{1}{CM} \sum_{i=1}^M \sum_{c=1}^C \hat{\gamma}_{i}\prs{c} \log \prs{\frac{1}{m_i} \sum_{j=1}^{m_i} f_c(X^i_j)}$ over the function $f$ in some space $\cF_0$. The $\hat{\gamma}_{i}\prs{c}$ is the label proportion of $c$-th entry in bag $i$ and $f_c$ is the $c$-th entry of the output of $f$. In practice, when $\cF_0$ is softmax composed with neural networks of certain architecture, the objective is optimized by stochastic gradient descent (SGD) with ``minibatches of bags". For a minibatch of size $B$, $B$ bags $b_{i_1}, \dots, b_{i_B}$ are sampled and SGD backpropagates the gradients of $\cL_{prop}\prs{b_{i_1}, \dots, b_{i_B}} = -\frac{1}{CB} \sum_{k=1}^B \sum_{c=1}^C \hat{\gamma}_{i_k}\prs{c} \log \prs{\frac{1}{m_{i_k}} \sum_{j=1}^{m_{i_k}} f_c(X^{i_k}_j)}$. In our experiments, we follow the code of \textcite{Tsai20LLPVAT} \footnote{\label{llpvat_code}https://github.com/kevinorjohn/LLP-VAT} and set $B=2$ (\textcite{DulacArnold2019DeepML} also use minibatches of bags but do not specify $B$). While optimizing neural networks with the KL loss on GPU nodes, the gradients of all data points in the minibatch of bags need to be stored in the GPU memory simultaneously. So KL loss can potentially exceed GPU memory when bag size increases. In this situation, we report \textit{Out of RAM} in the tables.

\subsection{LLPVAT}
Let $f$ and $\cL_{prop}$ be defined as in \ref{subsec:KL}. Let $D_{KL}$ denote the KL divergence. The LLPVAT algorithm of \textcite{Tsai20LLPVAT} computes the perturbed examples $\hat{x} = x + r_{adv}$ where $$r_{adv} = \argmax_{r:\norm{r}_2 \leq \epsilon} D_{KL} \prs{f(x) \mid\mid f(x+r_{adv})}.$$ Given a minibatch of $B$ bags $b_{i_1}, \dots, b_{i_B}$, their consistency loss is defined to be $$\cL_{cons}\prs{b_{i_1}, \dots, b_{i_B}} = \sum_{k=1}^B \frac{1}{\abs{b_{i_k}}} \sum_{x\in b_{i_k}} D_{KL} \prs{f(x) \mid\mid f(x+r_{adv})}.$$ For each minibatch in the $t$-th epoch, the LLPVAT algorithm updates the parameters of neural networks with the gradients of the loss $\cL\prs{b_{i_1}, \dots, b_{i_B}} = \cL_{prop}\prs{b_{i_1}, \dots, b_{i_B}} + w(t) \cL_{cons}\prs{b_{i_1}, \dots, b_{i_B}}$ where $w(t)$ is a ramp-up function for increasing the weight of consistency regularization. Following the LLPVAT paper, we set $\epsilon$ to 1 for both SVHN and EMNIST and set $\epsilon$ to 6 for CIFAR10. We follow the code of \textcite{Tsai20LLPVAT} \footnote{See footnote \ref{llpvat_code}} to implement $w(t)$ and set minibatch size $B$ to be 2. Like the KL loss, LLPVAT can potentially exceed GPU memory. In this situation, we report \textit{Out of RAM} in the tables.

\subsection{LLPGAN}
Let bags $b_i$ and $\Bar{\cL}_{prop}$ be defined as in subsection \ref{subsec:KL}. The LLPGAN model of \textcite{Liu19LLPGAN} consists of a generator $g$ and a discriminator $f$. The discriminator $f$ is a convolutional neural network and we denote its convolutional layers as $f_{conv}$. The generator $g$ maps a random noise to the image space and the discriminator maps an image to $\Delta^{C+1}$ where the fake images output by the generator are supposed to be classified as the $\prs{C+1}$-th class. Let $n$ be the total number of feature vectors in all bags, $i.e.$, $n=\sum_{i=1}^M m_i$ and $\cbs{Z_i^j, i\in\NN_{M}, j\in\NN_{m_i}}$ be random noise vectors sampled from a fixed distribution. The discriminator loss is defined as 
\begin{align*}
\Bar{\cL}_D &= -\sum_{i=1}^M \sum_{j=1}^{m_i} \frac{1}{m_i} \log\prs{\sum_{c\leq C} f_c(X^{i}_{j})} - \frac{1}{n}\sum_{i=1}^M \sum_{j=1}^{m_i} \log\prs{f_{C+1}(g(Z^{i}_{j}))} \\
& \qquad  
- \frac{1}{MC}\sum_{i=1}^M \sum_{c=1}^C\sum_{j=1}^{m_i}\frac{\hat{\gamma_i}\prs{c}}{m_i}\log\prs{f_c\prs{X^i_j}},
\end{align*}
where the last term $- \frac{1}{MC}\sum_{i=1}^M \sum_{c=1}^C\sum_{j=1}^{m_i}\frac{\hat{\gamma_i}\prs{c}}{m_i}\log\prs{f_c\prs{X^i_j}}$ is proposed as an upper bound of $\Bar{\cL}_{prop}$.
The generator loss is defined as $$\Bar{\cL}_G = \frac{1}{n}\sum_{i=1}^M \sum_{j=1}^{m_i}\norm{f_{conv}(X^i_j)-f_{conv}(g(Z^i_j))}_2^2.$$ 
Given a minibatch of $B$ bags $b_{i_1}, \dots, b_{i_B}$, the minibatch version of the discriminator loss is
\begin{align}
 {\cL}_D\prs{b_{i_1}, \dots, b_{i_B}} = & -\sum_{k=1}^B \sum_{j=1}^{m_{i_k}} \frac{1}{m_{i_k}} \log\prs{\sum_{c\leq C} f_c(X^{i_k}_{j})} \\
 &- \frac{1}{\sum_{k=1}^B m_{i_k}}\sum_{i=1}^B \sum_{j=1}^{m_{i_k}} \log\prs{f_{C+1}(g(Z^{i_k}_{j}))} \\
 &- \frac{1}{BC}\sum_{k=1}^B \sum_{c=1}^C\sum_{j=1}^{m_{i_k}}\frac{\hat{\gamma_i}\prs{c}}{m_{i_k}}\log\prs{f_c\prs{X^{i_k}_j}}.
\end{align} The minibatch version of generator loss is
$${\cL}_G\prs{b_{i_1}, \dots, b_{i_B}} = \frac{1}{\sum_{k=1}^B m_{i_k}}\sum_{i=1}^B \sum_{j=1}^{m_{i_k}}\norm{f_{conv}(X^{i_k}_j)-f_{conv}(g(Z^{i_k}_j))}_2^2.$$
So the training process of LLPGAN can be described as follows: in each epoch, for a minibatch $\prs{b_{i_1}, \dots, b_{i_B}}$, 
\begin{enumerate}
    \item Sample random noise $Z^{i_k}_j$ for $k\in\NN_B$ and $j\in\NN_{m_{i_k}}$.
    \item Fix $g$ and perform gradient descent on parameters of $f$ in $\cL_D$.
    \item Fix $f$ and perform gradient descent on parameters of $g$ in $\cL_G$.
\end{enumerate}
The code \footnote{See footnote \ref{LLPGAN_code}}  of the original paper implements $f$ as a neural network which outputs a vector in $\RR^C$ followed by a $\prs{C+1}$-way softmax with the $\prs{C+1}$-th input fixed to be 0. We follow this practice in our implementation of LLPGAN. We also follow the code of original paper and set $B$ to be 1. While LLPGAN could potentially exceed GPU memory while bag size increases as well, this did not happen in our experiments.


\subsection{Implementation Details of LLPFC}
 
 The performance of LLPFC algorithms benefit from re-partitioning of bags periodically. We randomly repartition the bags into groups every 20 epochs for WideResNet-16-4, ResNet-18, and LLPGAN discriminator and every 5 epochs for VGG16. 
 

\newcommand{\KL}{\mathtt{KL}}

\subsection{Experiments in Binary Setting}
We carry out an extra set of experiments with kernel methods on binary classification tasks, comparing against InvCal \cite{Rueping10invcal}, alter-$\propto$SVM \cite{Yu13PSVM}, and LMMCM \cite{binary}. We run our experiments on the exact same datasets used by \textcite{binary} and directly compare against the results presented in their paper. We implement LLPFC-uniform and LLPFC-approx with rbf kernel models and logistic loss by modifying the code provided by \textcite{binary} at \url{https://github.com/Z-Jianxin/Learning-from-Label-Proportions-A-Mutual-Contamination-Framework}. We run experiments in the same settings of \textcite{binary}. The model is solved by L-BFGS. We compute the kernel parameter by $\frac{1}{d * Var(X)}$ where $d$ is the number of features and $Var(X)$ is the variance of the data matrix. The regularization parameter $\lambda \in \{1, 10^{-1}, 10^{-2}, \ldots, 10^{-5}\}$ is chosen by 5-fold cross validation, using the empirical risk provided in Algorithm \ref{algo:LLPFC-uniform} and Algorithm \ref{algo:LLPFC-approx}, respectively. We evaluate the area under the ROC curve (AUC) and report the results in table \ref{tab:binary exp}. We bold the largest mean AUC for that experimental setting. Each of LLPFC-uniform, LLPFC-approx, and LMMCM achieves the highest AUC among all the methods in 5 settings. LLPFC-uniform also beats the three competitors from \textcite{binary} in 10 out 16 settings.


\begin{table*}[ht] \caption{AUC. Column header indicates bag size. \label{tab:binary exp}}
\resizebox{1.0\textwidth}{!}{
  \begin{tabular}{c c c c c c}
    \hline 
    Data set, LP dist & Method & 8 & 32 & 128 & 512 \\
    \hline
    \multirow{5}{*}{Adult, $\left[0, \frac{1}{2} \right]$} & InvCal & 
    0.8720 $\pm$ 0.0035 &
    0.8672 $\pm$ 0.0067 &
    0.8537 $\pm$ 0.0101 &
    0.7256 $\pm$ 0.0159
    \\
    & alter-$\propto$SVM& 
    0.8586 $\pm$ 0.0185 &
    0.7394 $\pm$ 0.0686 &
    0.7260 $\pm$ 0.0953 &
    0.6876 $\pm$ 0.1219 
    \\
    & LMMCM & 
    0.8728 $\pm$ 0.0019 &
    \textbf{0.8693 $\pm$ 0.0047} &
    \textbf{0.8669 $\pm$ 0.0041} &
    \textbf{0.8674 $\pm$ 0.0040} 
    \\
    & LLPFC-uniform & 
     \textbf{0.8751 $\pm$ 0.0022}& 
     0.8627 $\pm$ 0.0034& 
     0.8616 $\pm$ 0.0057& 
     0.8594 $\pm$ 0.0047
     
    \\
    & LLPFC-approx & 
     0.8676 $\pm$ 0.0042&
     0.8540 $\pm$ 0.0052&
     0.8509 $\pm$ 0.0094&
    0.8478 $\pm$ 0.0096
    \\
    
    \hline
    \multirow{5}{*}{Adult, $\left[\frac{1}{2}, 1 \right]$} & InvCal & 
    {0.8680 $\pm$ 0.0021} &
    0.8598 $\pm$ 0.0073 &
    0.8284 $\pm$ 0.0093 &
    0.7480 $\pm$ 0.0500
    \\
    & alter-$\propto$SVM& 
    0.8587 $\pm$ 0.0097 &
    0.7429 $\pm$ 0.1473 &
    0.8204 $\pm$ 0.0318 &
    0.7602 $\pm$ 0.1215
    \\
    & LMMCM & 
    0.8584 $\pm$ 0.0164 &
    {0.8644 $\pm$ 0.0052} &
    {0.8601 $\pm$ 0.0045} &
    {0.8500 $\pm$ 0.0186} 
    \\
     & LLPFC-uniform & 
     {0.8693 $\pm$ 0.0036}&
     \textbf{0.8666 $\pm$ 0.0047}&
     \textbf{0.8636 $\pm$ 0.0040}&
     \textbf{0.8587 $\pm$ 0.0136}
    \\
    & LLPFC-approx & 
    \textbf{0.8723 $\pm$ 0.0014} &
     0.8630 $\pm$ 0.0069&
     0.8560 $\pm$ 0.0103&
    0.8538 $\pm$ 0.0193
    \\
    
    \hline
    \multirow{5}{*}{MAGIC, $\left[0, \frac{1}{2} \right]$} & InvCal & 
    \textbf{0.8918 $\pm$ 0.0076} &
    0.8574 $\pm$ 0.0079 &
    0.8295 $\pm$ 0.0139 &
    0.8133 $\pm$ 0.0109
    \\
    & alter-$\propto$SVM& 
    0.8701 $\pm$ 0.0026 &
    0.7704 $\pm$ 0.0818 &
    0.7753 $\pm$ 0.0207 &
    0.6851 $\pm$ 0.1580 
    \\
    & LMMCM & 
    0.8909 $\pm$ 0.0077 &
    \textbf{0.8799 $\pm$ 0.0113} &
    \textbf{0.8753 $\pm$ 0.0157} &
    0.8734 $\pm$ 0.0092
    \\
     & LLPFC-uniform & 
     0.8575 $\pm$ 0.0644&
     0.8751 $\pm$ 0.0158&
     0.8715 $\pm$ 0.0066&
     \textbf{0.8761 $\pm$ 0.0157}
     
    \\
    & LLPFC-approx & 
    0.8829 $\pm$ 0.0135&
    0.8590 $\pm$ 0.0256&
    0.8721 $\pm$ 0.0054&
    0.8711 $\pm$ 0.0155
    \\
    
    \hline
    \multirow{5}{*}{MAGIC, $\left[\frac{1}{2}, 1 \right]$} & InvCal & 
    {0.8936 $\pm$ 0.0066} &
    0.8612 $\pm$ 0.0056 &
    0.8180 $\pm$ 0.0092 &
    0.8215 $\pm$ 0.0136
    \\
    & alter-$\propto$SVM& 
    0.8689 $\pm$ 0.0135 &
    0.8219 $\pm$ 0.0218 &
    0.8179 $\pm$ 0.0487 &
    0.7949 $\pm$ 0.0478 
    \\
    & LMMCM & 
    0.8911 $\pm$ 0.0083 &
    {0.8790 $\pm$ 0.0091} &
    {0.8684 $\pm$ 0.0046} &
    {0.8567 $\pm$ 0.0292} 
    \\
     & LLPFC-uniform & 
     {0.8985 $\pm$ 0.0054}&
     {0.8851 $\pm$ 0.0113}&
     {0.8844 $\pm$ 0.0101}&
     {0.8765 $\pm$ 0.0113}
     
    \\
    & LLPFC-approx & 
    \textbf{0.9011 $\pm$ 0.0034} &
    \textbf{0.8990 $\pm$ 0.0122} &
    \textbf{0.8882 $\pm$ 0.0088} &
    \textbf{0.8800 $\pm$ 0.0114}
    \\
  \end{tabular}
}
\end{table*}



\section{Proofs of Results from Section \ref{sec:calfun}}
\label{suppset:proof cali}
\subsection{Proof of Theorem \ref{thm:inv.cali}}
To prove Theorem \ref{thm:inv.cali}, we employ the calibration framework of \textcite{Steinwart2007HowTC}. The first lemma of this section establishes an instance of what \textcite{Steinwart2007HowTC} refers to as \emph{uniform calibration}, but in the LLN setting.
\begin{lemma}
\label{cali.inner.lemma} 
Let $\ell$ be a continuous strictly proper loss and $T$ be an invertible column-stochastic matrix. Let $L$ be the $0$-$1$ loss. Then $\forall \epsilon \in \RR_+, \exists \delta \in \mathbb{R}_+, s.t. \forall x \in \cX, \forall q \in \Delta^C,  $ $$ \cC_{\ell_{T}, P_T, x} (q) < \cC_{\ell_{T}, P_T, x}^* + \delta \implies \cC_{L, P, x}(q) < \cC_{L, P, x}^* + \epsilon.$$
\end{lemma}

\begin{proof}[Proof of Lemma \ref{cali.inner.lemma}]
Write $$\cC_{1, x}(q) :=  \cC_{L, P, x}(q) -  \cC_{L, P, x}^*$$ $$\cC_{2, x} (q):=  \cC_{\ell_{T}, P_T, x} (q) - \cC_{\ell_{T}, P_T, x}^*.$$
Let $\norm{\cdot}$ be an arbitrary norm on $\RR^C$, $k$ and $j \in \{1,\dots, C\}$ be such that $q_k = \max_i q_i$, and $\eta_j(x) = \max_i \eta_i(x)$. Then we have
\begin{align}
    \cC_{1, x}(q) &= \prs{1 - \eta_k(x)} - \prs{1 - \eta_j(x)} \\
    & = \eta_j(x) - \eta_k(x) \\
    & = \prs{\eta_j (x) - q_k} + \prs{q_k - \eta_k(x)} \\
    & = \norm{\eta(x)}_{\infty} - \norm{q}_{\infty}  + \prs{q_k - \eta_k(x)} \\ 
    & \leq \norm{\eta(x) - q}_{\infty} + \norm{\eta(x) - q}_{1} \\
    & \leq C_{\norm{\cdot}} \norm{\eta(x) - q},
\end{align}
where $C_{\norm{\cdot}}>0$ is a constant depending on the norm $\norm{\cdot}$ and the last inequality is implied by the equivalence of norms in finite dimensional space.\\
Hence,  $$\norm{\eta(x) - q} < \frac{\epsilon}{C_{\norm{\cdot}}} \implies \cC_{1, x}(q) < \epsilon.$$ So it suffices to prove $$\forall \delta_0 \in \mathbb{R}_+, \exists \delta \in \mathbb{R}_+ s.t. \forall x \in X, \forall q \in \Delta^C, \cC_{2, x} (q) < \delta \implies \norm{q - \eta(x)} < \delta_0.$$ To this end, assume its negation, $$ \exists \delta_0 \in \mathbb{R}_+, \forall \delta \in \mathbb{R}_+ s.t. \exists x \in X, \exists q \in \Delta^C, \cC_{2, x} (q) < \delta \text{ and } \norm{q - \eta(x)} \geq \delta_0.$$ 
Let $\delta_n = \frac{1}{n}$ for each  $n \in \NN$. We can obtain a sequence $\cbs{\prs{q_i, \eta(x_i)}}_{i=1}^{\infty} \subset K:= \cbs{(q, \eta) \in \Delta^C \times \Delta^C : \norm{q - \eta} \geq \delta_0}$ such that $\cC_{2, x_i} (q_i) < \delta_i$ for all $i \in \NN$.
As $K$ is compact, we can extract a convergent subsequence $\cbs{\prs{q_{i_k}, \eta(x_{i_k})}}_{k=1}^{\infty}$. Let $$\lim_k \prs{q_{i_k}, \eta(x_{i_k})} = \prs{q^*, \eta^*} \in K.$$ Write $$\cC_{2}  (q, \eta) := \sum_{c=1}^C \prs{T\eta}_c \prs{\ell\prs{Tq, c} - \ell\prs{T\eta, c} },$$ so $\cC_2$ is continuous and $\cC_{2, x}(q) = \cC_2(q, \eta(x)).$ Therefore, $$0 = \lim_k \cC_{2, x_{i_k}} (q_{i_k}) = \lim_k \cC_2\prs{q_{i_k}, \eta(x_{i_k})} = \cC_2 \prs{q^*, \eta^*},$$ which contradicts the strict properness of $\ell$ since $q^* \ne \eta^*$.
\end{proof}

This lemma establishes what may be viewed as a pointwise notion of consistency: For each fixed $x \in \cX$, the target excess 0/1-inner risk (defined $w.r.t.$ $P$) can be made arbitrarily small by making the surrogate excess $\ell_{T}$-inner risk (defined $w.r.t.$ $P_T$) sufficiently small. 

Now let $\epsilon \in \brs{0, \infty}$, and denote $A(\epsilon) :=$ $$\cbs{\delta \in \brs{0, \infty}: \forall x \in \cX, \forall q \in \Delta^C,\, \cC_{\ell_{T}, P_T, x} (q) < \cC_{\ell_{T}, P_T, x}^* + \delta \implies \cC_{L, P, x}(q) < \cC_{L, P, x}^* + \epsilon}.$$ Define the function $$\delta: \brs{0, \infty} \rightarrow \brs{0, \infty} , \ \ \delta(\epsilon) = \sup_{\delta \in A(\epsilon)} \delta.$$ The following properties immediately follow from the definition and  Lemma \ref{cali.inner.lemma}: \begin{itemize}
    \item $\delta(0) = 0$ and $\delta(\epsilon) > 0$ if $\epsilon > 0$. 
    \item $\delta (\cdot)$ is monotone non-decreasing.
    \item $A(\epsilon) = \brs{0, \delta(\epsilon)}$
\end{itemize}

The function $\delta$ is a reasonable candidate for the sought after function $\theta$, but it is not necessarily invertible since it might not be strictly increasing. To address this we introduce the following.

\begin{definition} 
Let $I \subset \RR$ be an interval and  let $g: I \rightarrow \brs{0, \infty}$ be a function. Then the \emph{Fenchel-Legendre biconjugate} $g^{**}: I \rightarrow \brs{0, \infty}$ of $g$ is the largest convex function $h: I \rightarrow \brs{0, \infty}$ satisfying $h \leq g$.
\end{definition}

We are now prepared to prove theorem \ref{thm:cali.noise} of which theorem \ref{thm:inv.cali} is a direct corollary.

\newcommand{\Biconj}{\delta_{\vert_{\brs{0, 1}}}^{**}}
\begin{theorem} \label{thm:cali.noise}
Let $\ell$ be a continuous strictly proper loss and T be an invertible column-stochastic matrix. Let $L$ be the $0$-$1$ loss and let $\delta(\cdot)$ be the function defined above. Assume $\cR_{\ell_{T}, P_T}^* < \infty$. Then for all $P$, $$\forall f \in \cF, \Biconj \prs{\cR_{L, P}(f) - \cR^*_{L, P}} \leq \cR_{\ell_{T}, P_T}(f) - \cR_{\ell_{T}, P_T}^*$$ where $\Biconj$ denotes the Fenchel-Legendre biconjugate of the restriction $\delta_{\vert \brs{0, 1}}$.
\end{theorem}

\begin{proof}[Proof of Theorem \ref{thm:cali.noise}]
Write $$\cC_{1, x}(q) :=  \cC_{L, P, x}(q) -  \cC_{L, P, x}^*$$ $$\cC_{2, x} (q):=  \cC_{\ell_{T}, P_T, x} (q) - \cC_{\ell_{T}, P_T, x}^*.$$ Then, 
$$\forall x \in \cX, \forall p, q \in \Delta^C, \cC_{2,x}(p) < \delta(\cC_{1,x}(q)) \implies \cC_{1,x}(p) < \cC_{1,x}(q).$$ By letting $p=q$, we have $\forall x \in \cX, \forall q \in \Delta^C, \cC_{2,x}(q) \geq \delta(\cC_{1,x}(q)).$ Fix $f \in \cF$, 
\begin{align}
    \Biconj \prs{\cR_{L, P}(f) - \cR^*_{L, P}} & = \Biconj \prs{ \int_{\cX}  \cC_{1,x}\prs{f(x)} d P_X(x)} \label{main.p_minimizable} \\
    & \leq \int_{\cX} \Biconj \prs{\cC_{1,x}\prs{f(x)}} d P_X(x) \label{main.first} \\ 
    & \leq \int_{\cX}  \cC_{2,x}\prs{f(x)} d P_X(x)  \label{main.second} \\
    & = \cR_{\ell_{T}, P_T}(f) - \cR_{\ell_{T}, P_T}^* \label{supp.surr.p_min}
\end{align}
\eqref{main.p_minimizable} follows the fact $\cR_{L, P}(f) = \int_{\cX} \cC_{L, P, x}\prs{f(x)} dP_X(x)$ and $\cR^*_{L, P} = \int_{\cX} \cC_{L, P, x}\prs{\eta(x)} dP_X(x) = \int_{\cX} \cC_{L, P, x}^* dP_X(x)$. \eqref{main.first} is implied by Jensen's inequality and the convexity of $\Biconj$ and \eqref{main.second} by the fact $\Biconj (\cdot) \leq \delta(\cdot)$. \eqref{supp.surr.p_min} follows  $\cR_{\ell_{T}, P_T}(f) = \int_{\cX} \cC_{\ell_{T}, P_T, x}\prs{f(x)} dP_X(x)$ and $\cR^*_{\ell_{T}, P_T} = \int_{\cX} \cC_{\ell_{T}, P_T, x}\prs{\eta(x)} dP_X(x) = \int_{\cX} \cC_{\ell_{T}, P_T, x}^* dP_X(x) < \infty$.
\end{proof}

\subsection{Proof of Proposition \ref{theorem: cali_lowerb}}
\label{subsec:proof_yw}

For the reader's convenience, we restate Proposition \ref{theorem: cali_lowerb} below:
\begin{proposition}
  \label{theorem: convex LB - calib func - FC log loss}
    Let $T \in \mathbb{R}^{C\times C}$ be an invertible, column-stochastic matrix.
  Define $\underline{\theta}_T : \brs{0, \infty} \rightarrow \brs{0, \infty}$ by
  \[
    \underline{\theta}_T(\epsilon)
    =
    \frac{1}{2}
\frac{
  \epsilon^2
  }
{\|T^{-1}\|_1^2}.
  \]
  Then for all $x \in X$ and $q \in \Delta^C$, we have
  \[
    \underline{\theta}_T(\cC_{L, P, x}(q) -  \cC_{L, P, x}^*)
      \le
    \cC_{\ell_{T}, P_T, x} (q) - \cC_{\ell_{T}, P_T, x}^*.
  \]
\end{proposition}

Below, let $\ell$  denote the 
\textit{log loss} $\ell^{log}\prs{q, c} = - \log q_c$
and $L$ denote the $0-1$ loss:
\[L: \Delta^C \times Y \rightarrow \cbs{0, 1}, \qquad L\prs{q, c} = \mathbbm{1}_{\cbs{c \ne \min\{\argmax q\} }}.\]

To proceed with the proof, we first introduce some notations and useful results.
For $p,q \in \Delta^C$, define
\begin{equation}
  \label{equation-definition: Bayes risk - log loss}
  \underline{\ell}(q,p)
  :=
\mathbb{E}_{y \sim p} \ell(q, y)
=
\sum_{i=1}^C -p_i \log q_i.
\end{equation}
The above quantity is often referred to as the \emph{cross entropy} of $q$ relative to $p$.
Next, since the log loss is proper \cite{Williamson2016ProperComp}, we have
\begin{equation}
  \label{the log loss is proper}
  \inf_{q \in \Delta^C} 
  \underline{\ell}(q,p)
  =
  \underline{\ell}(p,p).
\end{equation}

The \emph{Kullback-Leibler (KL) divergence} between $p,q \in \Delta^c$ is defined as
\begin{equation}
  \label{equation-definition: KL divergence}
  \KL(p\|q) 
  = \underline{\ell}(p,q)
  -
  \underline{\ell}(p,p).
\end{equation}
In the literature, the KL divergence is often presented as
$
  \KL(p\|q) = \sum_{i=1}^C p_i \log \left( \frac{p_i}{q_i}\right) 
  $ which is easily shown to be equivalent to \eqref{equation-definition: KL divergence}.
  We now rewrite the right-hand side of 
  the inequality in Proposition~\ref{theorem: convex LB - calib func - FC log loss} in terms of the KL divergence: 
  
\begin{lemma} \label{YW_lemma:KL}
  Let $p := P(\cdot | x) \in \Delta^C$.
  Then
  \begin{equation}
    \label{equation: calibration function - upper}
    \cC_{\ell_{T}, P_T, x} (q) 
    -
    \cC_{\ell_{T}, P_T, x}^*
    = \KL(Tp \| Tq).
  \end{equation}
\end{lemma}

\begin{proof}[Proof of Lemma \ref{YW_lemma:KL}]
  
  By definition, we have $Tp = P_T(\cdot |x)$.
  Unwinding the definitions, we have
  \[
    \cC_{\ell_{T}, P_T, x} (q) 
    =
\mathbb{E}_{y \sim P_T( \cdot | x)} \ell_{T}(q, y)
=
\mathbb{E}_{y \sim Tp} \ell(Tq, y)
=
\underline{\ell}(Tq,Tp).
  \]
  Furthermore,
  \[
    \cC_{\ell_{T}, P_T, x}^*
    =
    \inf_{q \in \Delta^C}
\mathbb{E}_{y \sim P_T( \cdot | x)} \ell_{T}(q, y)
=
    \inf_{q \in \Delta^C}
\underline{\ell}(Tq,Tp) 
=
\underline{\ell}(Tp,Tp)
  \]
  where the last equality follows from \eqref{the log loss is proper}.
  Now, \eqref{equation: calibration function - upper} follows immediately from \eqref{equation-definition: KL divergence}.
\end{proof}

Next, we focus on the term $\cC_{L, P, x}(q) -  \cC_{L, P, x}^*$ on the left-hand side in Proposition~\ref{theorem: convex LB - calib func - FC log loss}.
Analogous to 
  \eqref{equation-definition: Bayes risk - log loss}, we define
  \begin{equation}
    \label{equation-definition: Bayes risk - zo loss}
  \underline{L}(q,p)
  :=
\mathbb{E}_{y \sim p} L(q, y)
=
\sum_{c=1}^C p_c \mathbbm{1}_{\cbs{c \ne \min\{\argmax q\} }}
=
1 - p_{\min \{\argmax q\}}.
  \end{equation}
  The $0-1$ loss is also proper and 
  \begin{equation}
    \label{equation-definition: Bayes minimum risk - zo loss}
    \inf_{q \in \Delta^C} \underline{L}(q,p) = \underline{L}(p,p) =
1 - p_{\min \{\argmax p\}}
    = 1 - \max_c p_c.
  \end{equation}
  Unwinding the definition, we have
\[
\cC_{L, P, x}(q) 
=
\mathbb{E}_{y \sim P( \cdot | x)} L(q, y)
=
\underline{L}(q,p)
\]
and
\[
\cC_{L, P, x}^*
=
    \inf_{q \in \Delta^C}
    \mathbb{E}_{y \sim P( \cdot | x)} L(q, y)
    =
    \inf_{q \in \Delta^C}
    \underline{L}(q,p)
    =
    \underline{L}(p,p).
\]
Thus,
\begin{equation}
  \label{equation: calibration function - lower}
\cC_{L, P, x}(q) 
-
\cC_{L, P, x}^*
=
\underline{L}(q,p)
-
    \underline{L}(p,p).
\end{equation}

Thus, by \eqref{equation: calibration function - upper} and \eqref{equation: calibration function - lower}, we only need to 
focus on comparing
$
\underline{L}(q,p)
-
    \underline{L}(p,p)
    $
    with $\KL(Tp \| Tq)$.
    This is facilitated by the $1$-norm $\|\cdot \|_1$
    and the next two results.
The first is by 
\textcite{pinsker1964information}:
\begin{theorem}
  [Pinsker inequality]
  \label{theorem: Pinsker inequality}
  Let $\|\cdot \|_1$ be the $1$-norm on $\mathbb{R}^C$. Then for all $p,q \in \Delta^C$, we have
  \[
    \KL(p\|q) \ge 
    \frac{1}{2} \| p - q\|_1^2.
  \]
\end{theorem}
The second one is widely-known in the literature. For the sake of completeness, we provide a proof using our notations:
 \begin{lemma}
    \label{lemma: L1 norm upper bounds ZO divergence}
   Let $p,q \in \Delta^C$ be arbitrary. Then 
    $
    \| p - q\|_1
    \ge 
    \underline{L}(q,p) - \underline{L}(p,p).
  $
\end{lemma}
\begin{proof}[Proof of Lemma \ref{lemma: L1 norm upper bounds ZO divergence}]
    Let $i := \min \{\argmax p\}$ and $j := \min \{\argmax q\}$.
    Then by \eqref{equation-definition: Bayes risk - zo loss} and \eqref{equation-definition: Bayes minimum risk - zo loss}, we have
    \[
    \underline{L}(q,p) - \underline{L}(p,p)
    =
    1 - p_{j} - (1- p_{i})
    =
    p_{i} - p_{j}.
    \]
    On the other hand, note that 
    \begin{align}
      \|p-q \|_1  &= \sum_{i=1}^C |p_i - q_i|
      \\
                  &\ge |p_{i} - q_{i}| + | q_{j} - p_{j}|
                  \\
                  & \ge
                  |p_{i} - p_{j} + q_{j} - q_{i}| \qquad \because \mbox{triangle inequality}
                  \\
                  &
                  =
                  p_{i} - p_{j} + q_{j} - q_{i}
                  \qquad \because 
                  p_{i} - p_{j} \ge 0 \mbox{ and }  q_{j} - q_{i} \ge 0
                  \\
                  &
                  \ge 
                  p_{i} - p_{j}
    \end{align}
    as desired.
  \end{proof}

  Finally, we need one more result to take into account the presence of the stochastic matrix $T$ when applying Pinsker inequality to lower bound $\KL(Tp\|Tq)$:
 \begin{lemma}
    \label{lemma: operator norm inverse infimum}
  Let $M \in \mathbb{R}^{C \times C}$ be a matrix and let $\| \cdot \|$ be a norm on $\mathbb{R}^C$.
    Suppose that $M$ is non-singular.
    Then \[
      \inf_{x \in \mathbb{R}^C: x \ne 0}
      \frac{\|Mx\|}{\|x\|}
      =
      \frac{1}{\|M^{-1}\|}.
    \]
  \end{lemma}
  \begin{proof}[Proof of Lemma \ref{lemma: operator norm inverse infimum}]
    We begin by rewriting the infimum as the reciprocal of a supremum:
    \[
      \inf_{x \in \mathbb{R}^C : x \ne 0} \frac{\|Mx\|}{\|x\|}
      =
      \left(
        \sup_{x \in \mathbb{R}^C : x \ne 0} \frac{\|x\|}{\|Mx\|}
      \right)^{-1}.
    \]
    Next, applying the change of variables $x = M^{-1}y$, we have
    \[
        \sup_{x \in \mathbb{R}^C : x \ne 0} \frac{\|x\|}{\|Mx\|}
        =
        \sup_{y \in \mathbb{R}^C : y \ne 0} \frac{\|M^{-1} y\|}{\|y\|}
        =
        \|M^{-1}\|
    \]
    where the last equality holds by definition.
  \end{proof}

\begin{proof}[Proof of Proposition \ref{theorem: cali_lowerb}]

We are now ready to conclude the proof. Putting it all together, we have
  \begin{align}
    &\cC_{\ell_{T}, P_T, x} (q) 
    -
    \cC_{\ell_{T}, P_T, x}^*
    \\
    &=\KL(Tp \| Tq)
    \qquad \because \mbox{Equation \eqref{equation: calibration function - upper}}
    \\
    &\ge \frac{1}{2} \| Tp - Tq \|_1^2
    \qquad \because \mbox{Theorem \ref{theorem: Pinsker inequality}, Pinsker inequality}
    \\
    &= \frac{1}{2} \| T(p - q) \|_1^2
    \\
    &\ge
    \frac{1}{2}
\frac{\|p-q\|_1^2}
{\|T^{-1}\|_1^2}
    \qquad \because \mbox{Lemma \ref{lemma: operator norm inverse infimum}}
    \\
    &\ge
    \frac{1}{2}
\frac{
  (\underline{L}(q,p) - \underline{L}(p,p))^2
  }
{\|T^{-1}\|_1^2}.
    \qquad \because \mbox{
      Lemma
    \ref{lemma: L1 norm upper bounds ZO divergence}
  }
  \\
    &=
    \frac{1}{2}
\frac{
  (
\cC_{L, P, x}(q) 
-
\cC_{L, P, x}^*
  )^2
  }
{\|T^{-1}\|_1^2}
\qquad \because \mbox{Equation \eqref{equation: calibration function - lower}}
  \end{align}
  as desired.
  This concludes the proof of Proposition~\ref{theorem: convex LB - calib func - FC log loss}.
\end{proof}

\section{Remakrs for Section \ref{sec:mntm}} 
\subsection{Remarks on the Setting of LMNTM} \label{subsec:remarks lmntm}
Instead of letting $\prs{X, \Tilde{Y}}\stackrel{i.i.d.}{\sim} P_{T_i}$, which is a more common assumption, we choose the setting described in Section \ref{subset: lmntm_setup} because it fits LLP more naturally. When reducing LLP to LLN, a bag in group $i$ is modeled as a collection of data points sampled from $P_{T_i}(\cdot\mid c)$. If we assume all data points in group $i$ are sampled $i.i.d.$ from $P_{T_i}$, then we need $\prs{n_{i,1}, n_{i,2}, \dots, n_{i,C}}$, the size of bags in group $i$,  to follow a multinomial distribution, which is too restrictive. Our current setting is more flexible and allows $n_{i,c}$ to be either deterministic or random.



\section{Proofs for Section \ref{sec:mntm}} \label{suppsec:proofs mntm}

\subsection{Proof of Theorem \ref{llp.cali}}
\begin{proof}[Proof of Theorem \ref{llp.cali}]
By Theorem \ref{thm:cali.noise}, $\forall i \in \NN, \exists $ a strictly increasing continuous function $\theta_i$ with $\theta_i (0) = 0$ and $$\theta_i \prs{\cR_{L, P}(f) - \cR^*_{L, P g}} \leq \cR_{\ell_{T_i}, P_{T_i}}(f) - \cR_{\ell_{T_i}, P_{T_i}}^*.$$ Then,  $$\sum_{i=1}^{N} w_i \theta_i \prs{\cR_{L, P}(f) - \cR^*_{L, P}} \leq \sum_{i=1}^{N} w_i \prs{\cR_{\ell_{T_i}, P_{T_i}}(f) - \cR_{\ell_{T_i}, P_{T_i}}^*} = \widetilde{\cR}_{\ell, P, \cT}(f) - \widetilde{\cR}_{\ell, P, \cT}^* $$  The last equality is implied by the fact $\cR_{\ell_{T_i}, P_{T_i}}^* =  \cR_{\ell_{T_i}, P_{T_i}} (\eta(x)) < \infty$. Let $\theta = \sum_{i=1}^{N} w_i \theta_i$ which is clearly continuous and satisfies $\theta (0) = 0$.
\end{proof}

\subsection{Proof of Theorem \ref{thm: geb}}

Now we introduce a sequence of lemmas to prove the generalization error bound.

\begin{lemma} \label{rad.lips.lemma}
Let $\cG \subset \psi \circ \cF$ s.t. $\sup_{x\in X, g \in \cG} \norm{g\prs{x}}_2 \leq A$ for some constant $A$. Let $N\in\NN$ and $\cT = \cbs{T_i}_{i=1}^N$ be a sequence of invertible column-stochastic matrices. Fix $\prs{w_1, w_2, \dots, w_N}^{tr} \in \Delta^N$ and $n_{i,c} \in \NN$ for each $i\in\NN_N$ and $c\in\cY$. Let $S=\cbs{X_{i, c, j}: i = \NN_N, c \in \cY, j = \NN_{n_{i,c}}}$ where each $X_{i, c, j}$ is drawn from the class conditional distribution $P_{T_i}(\cdot\mid c)$ and all $X_{i, c, j}$'s are independent. $\forall i \in \NN_N$ and $c\in\cY$, let $\alpha_i \in \mathring{\Delta}^C$ $s.t.$ $\alpha_i(c)$ = $P_{T_i} (\tilde{Y}=c)$. Let $\ell$ be a proper loss s.t. $\forall i, c$ the function $\lambda_{\ell_{T_i}} \prs{\cdot ,c}$ is Lipschitz $w.r.t.$ the $2$-norm. Write 
$$\hat{\cR}_{w, S}(g) := \sum_{i=1}^N w_i \sum_{c=1}^C \frac{\alpha_i(c)}{n_{i, c}} \sum_{j=1}^{n_i} \lambda_{\ell_{T_i}} \prs{g\prs{X_{i,c,j}}, c)}$$ and $$\widetilde{\cR}(g) := \widetilde{R}_{\ell, P, \cT} \prs{\psi^{-1} \circ g} = \EE \brs{\hat{\cR}_{w, S}(g)}.$$ 
Then $\forall \delta \in (0, 1]$, with probability at least $1-\delta$ $w.r.t.$ to the draw of $S$,  
\begin{align}
\sup_{g\in\cG} \abs{\hat{\cR}_{w, S}(g) - \widetilde{\cR}(g)} 
\leq &\sqrt{2 \log \frac{2}{\delta} \sum_{i=1}^N \sum_{c=1}^C \frac{w_i^2 \alpha_i^2(c)}{n_{i, c}} \prs{\abs{\lambda_{\ell_{T_i}}}A + \abs{\lambda_{\ell_{T_i}}}_0}^2} 
\\
&\qquad +2 \EE_{S, \epsilon_{i,c,j}}\brs{\sup_{g\in\cG} \sum_{i=1}^N w_i \sum_{c=1}^C \frac{\alpha_i(c)}{n_{i,c}} \sum_{j=1}^{n_{i,c}} \epsilon_{i,c,j} \lambda_{\ell_{T_i}} (g(X_{i, c,j}), c)},
\end{align} where $\epsilon_{i,c, j}, i = \NN_N, c \in\cY, j \in \NN_{n_{i,c}}$ are i.i.d. Rademacher random variables, $\abs{\lambda_{\ell_{T_i}}}_0 = \max_{c} \abs{\lambda_{\ell_{T_i}}(0, c)}$, and $\abs{\lambda_{\ell_{T_i}}}$ is the smallest real number such that it is a Lipschitz constant of $\lambda_{\ell_{T_i}} \prs{\cdot ,c}$, $\forall i, c$.
\end{lemma}
\begin{proof}
Write 
$$\xi(S) := \sup_{g\in\cG} \abs{\hat{R}_{w, S}(g) - \widetilde{R}(g)},$$
$$\xi^+(S) := \sup_{g\in\cG} \hat{R}_{w, S}(g) - \widetilde{R}(g),
\quad \mbox{ and } \quad \xi^-(S) := \sup_{g\in\cG} -\prs{\hat{R}_{w, S}(g) - \widetilde{R}(g)}.$$
We will show that the same bound on $\xi^+(S)$ and $\xi^-(S)$ holds with probability at least $1-\frac{\delta}{2}$. Combining these bounds gives the desired bound on $\xi(S)$. We first consider $\xi^+(S)$. The analysis for $\xi^-(S)$ is identical. By definition, 
$$\xi^+(S) = \sup_{g\in\cG} \sum_{i=1}^N w_i\brs{\sum_{c=1}^C \frac{\alpha_i(c)}{n_{i,c}} \sum_{j=1}^{n_i} \lambda_{\ell_{T_i}} \prs{g(X_{i, c, j}), c} - \cR_{\ell_{T_i}, P_{T_i}} \prs{\psi^{-1} \circ g}}.$$
We first use the Bounded Difference Inequality \cite{mcdiarmid1989method} to bound $\xi^+(S) - \EE \xi^+(S)$. Substitute $X_{i, c, j}$ with arbitrary $X'_{i, c, j}$ and $\xi^+(S)$ changes by at most $\sup_{g\in\cG}\frac{w_i\alpha_i(c)}{n_{i,c}} \abs{\lambda_{\ell_{T_i}} \prs{g(X_{i, c, j}), c} - \lambda_{\ell_{T_i}} \prs{g(X'_{i, c, j}), c}}$. Furthermore,
\begin{align}
    \abs{\lambda_{\ell_{T_i}} \prs{g(X_{i, c, j}), c}} & \leq \abs{\lambda_{\ell_{T_i}} \prs{g(X_{i, c, j}), c}  - \lambda_{\ell_{T_i}} (0, c)} + \abs{\lambda_{\ell_{T_i}} (0, c)} \\
    & \leq \abs{\lambda_{\ell_{T_i}}} \norm{g(X_{i, c, j})} + \abs{\lambda_{\ell_{T_i}}}_0 \\
    & \leq \abs{\lambda_{\ell_{T_i}}} A + \abs{\lambda_{\ell_{T_i}}}_0.
\end{align}
Hence, $$\sup_{g\in\cG}\frac{w_i\alpha_i(c)}{n_{i,c}} \abs{\lambda_{\ell_{T_i}} \prs{g(X_{i, c, j}), c} - \lambda_{\ell_{T_i}} \prs{g(X'_{i, c, j}), c}}  \leq 2 \frac{w_i\alpha_i(c)}{n_{i,c}} \prs{\abs{\lambda_{\ell_{T_i}}} A + \abs{\lambda_{\ell_{T_i}}}_0}.$$ 
By the Bounded Difference Inequality, with probability at least $1-\frac{\delta}{2}$, 
$$\xi^+(S) - \EE \xi^+(S) \leq \sqrt{2\log{\frac{2}{\delta}} \sum_{i=1}^N \sum_{c=1}^C \frac{w_i^2 \alpha_i(c)^2}{n_{i,c}} \prs{\abs{\lambda_{\ell_{T_i}}} A + \abs{\lambda_{\ell_{T_i}}}_0}^2}.$$
It remains to bound $\EE \xi^+(S)$. Let $S' = \cbs{X'_{i,c,j}, : i = \NN_N, c\in\cY, j = \NN_{n_i}}$ where every pair of $X'_{i,c,j}$ and $X_{i,c,j}$ are i.i.d. and all $X'_{i,c,j}$'s are independent. Hence, 
\begin{align}
    \EE_{S} \brs{ \xi (S)} &= \EE_{S} \brs{\sup_{g\in\cG} \sum_{i=1}^N w_i \sum_{c=1}^C \frac{\alpha_i(c)}{n_{i,c}} \sum_{j=1}^{n_{i,c}} \lambda_{\ell_{T_i}} (g(X_{i, c,j}), c) - \sum_{i=1}^{N} w_i \cR_{\ell_{T_i}, P_{T_i}}(\psi^{-1} \circ g)} \\
    &= \EE_{S} \brs{\sup_{g\in\cG} \prs{\hat{\cR}_{w, S}(g) - \EE_{S'} \hat{\cR}_{w, S'}(g)}} \\
    & \leq \EE_{S} \EE_{S'} \brs{\sup_{g \in \cG} \prs{\hat{\cR}_{w, S}(g) - \hat{\cR}_{w, S'}(g)}} \label{rad.lips.lemma.shadow.jensen} \\
    & = \EE_{S, S'}\brs{\sup_{g\in\cG} \sum_{i=1}^N w_i \sum_{c=1}^C \frac{\alpha_i(c)}{n_{i,c}} \sum_{j=1}^{n_{i,c}} \prs{\lambda_{\ell_{T_i}} (g(X_{i, c,j}), c)- \lambda_{\ell_{T_i}} (g(X'_{i, c,j}), c)}} \\ 
    & = \EE_{S, S',\epsilon_{i,c,j}}\brs{\sup_{g\in\cG} \sum_{i=1}^N w_i \sum_{c=1}^C \frac{\alpha_i(c)}{n_{i,c}} \sum_{j=1}^{n_{i,c}} \epsilon_{i,c,j} \prs{\lambda_{\ell_{T_i}} (g(X_{i, c,j}), c)- \lambda_{\ell_{T_i}} (g(X'_{i, c,j}), c)}}  \label{rad.lips.lemma.shadow.sym} \\
    & \leq \EE_{S, S',\epsilon_{i,c,j}}\brs{\sup_{g\in\cG} \sum_{i=1}^N w_i \sum_{c=1}^C \frac{\alpha_i(c)}{n_{i,c}} \sum_{j=1}^{n_{i,c}} \epsilon_{i,c,j} \lambda_{\ell_{T_i}} (g(X_{i, c,j}), c)} \\
    &\qquad + \EE_{S, S',\epsilon_{i,c,j}}\brs{\sup_{g\in\cG} \sum_{i=1}^N w_i \sum_{c=1}^C \frac{\alpha_i(c)}{n_{i,c}} \sum_{j=1}^{n_{i,c}} \epsilon_{i,c,j} \lambda_{\ell_{T_i}} (g(X'_{i, c,j}), c)} \label{rad.lips.lemma.shadow.last.ineq} \\
    & = 2 \EE_{S, \epsilon_{i,c,j}}\brs{\sup_{g\in\cG} \sum_{i=1}^N w_i \sum_{c=1}^C \frac{\alpha_i(c)}{n_{i,c}} \sum_{j=1}^{n_{i,c}} \epsilon_{i,c,j} \lambda_{\ell_{T_i}} (g(X_{i, c,j}), c)}.
\end{align}
\eqref{rad.lips.lemma.shadow.jensen} is implied by the convexity of $\sup_{g\in\cG}$ and Jensen's inequality. The equality in \eqref{rad.lips.lemma.shadow.sym} holds because $X'_{i, c,j}$ and $X_{i, c,j}$ are $i.i.d.$ and $\epsilon_{i,c,j}$ is symmetric. \eqref{rad.lips.lemma.shadow.last.ineq} can be justified by the elementary property of supremum and symmetry of $\epsilon_{i, c, j}$. 
\end{proof}

We need the next two lemmas to get rid of the $\lambda_{\ell_{T_i}}$'s when the set $\cV \subset \RR^C$.
\begin{lemma} \label{lemma: multi.rad.1}
Let $\mathcal{H}$ be a set of functions from $\mathcal{X}$ to $\mathbb{R}^C$, let $\phi$ be a function from $\mathcal{H}$ to $\mathbb{R}$, let $a$ be a positive real number, and let $\lambda: \mathbb{R}^C \rightarrow \mathbb{R}$ be a Lipschitz function $w.r.t.$ the norm $\norm{\cdot}_2$. We denote the Lipschitz constant of $\lambda$ by $\abs{\lambda}$. Then,
$$\mathbb{E}_{\epsilon} \sup_{f \in \mathcal{H}} \epsilon a \lambda(f(x)) + \phi(f) \leq \mathbb{E}_{\epsilon_1, \dots, 
\epsilon_C} \sup_{f \in \mathcal{H}} \sqrt{2} a |\lambda| \sum_{c=1}^C \epsilon_{c} f_c(x) + \phi(f)$$ 
where $\epsilon, \epsilon_1, \dots, \epsilon_C$ are independent Rademacher variables and $f_c(x)$ denotes the c-th entry of $f(x)$.
\end{lemma}

\begin{proof}
By 
Proposition 1 of \textcite{maurer2016vector}, 
\begin{equation} \label{vec.rad.ineq}
    \forall M \in \mathbb{N}, \forall v \in \mathbb{R}^M, \|v\|_2 \leq \sqrt{2}\mathbb{E}_{\epsilon_m} \left| \sum_{m=1}^M v_m \epsilon_m \right|.
\end{equation}
Fix $\delta > 0$, then $\exists f^*, g^* \in \cF$,
\begin{align}
    &   \ \ \ \ 2 \brs{\mathbb{E}_{\epsilon} \sup_{f \in \mathcal{H}} \epsilon a \lambda (f(x)) + \phi(f)} - \delta   \\
    & = \sup_{f, g \in \mathcal{H}} \brs{a \lambda(f(x)) + \phi(f) -  a \lambda(g(x)) + \phi(g)} - \delta   \\
    & < a (\lambda(f^*(x)) - \lambda(g^*(x))) + \phi(f^*) + \phi(g^*) \label{supp.lem.multi.rad.1.step1} \\
    & \leq a |\lambda| \norm{f^*(x) - g^*(x)}_2 + \phi(f^*) + \phi(g^*) \\
    & \leq \EE_{\epsilon_c} \sqrt{2} a |\lambda| \abs{\sum_{c=1}^C \epsilon_c (f^*_c(x) - g^*_c(x))} + \phi(f^*) + \phi(g^*) \label{supp.lem.multi.rad.1.use.vec.rad.ineq} \\
    & \leq \mathbb{E}_{\epsilon_c} \sup_{f, g \in \mathcal{H}} \brs{ \sqrt{2} a |\lambda|\abs{\sum_{c=1}^C \epsilon_c (f_c(x) - g_c(x))}  + \phi(f) + \phi(g) } \\
    & = \mathbb{E}_{\epsilon_c} \sup_{f \in \mathcal{H}} \left[\sqrt{2} a |\lambda|\sum_{c=1}^C \epsilon_c f_c(x)  + \phi(f) \right] + \mathbb{E}_{\epsilon_c} \sup_{g \in \mathcal{H}} \left[-\sqrt{2} a |\lambda|\sum_{c=1}^C \epsilon_c g_c(x)  + \phi(g) \right] \label{supp.lem.multi.rad.1.drop.abs} \\
    & = 2\mathbb{E}_{\epsilon_c} \sup_{f \in \mathcal{H}} \left[\sqrt{2} a |\lambda|\sum_{c=1}^C \epsilon_c f_{c}(x)  + \phi(f) \right]
\end{align}
The existence of $f^*, g^*$ satisfying the inequality in step \eqref{supp.lem.multi.rad.1.step1} is guaranteed by the definition of supremum. Step \eqref{supp.lem.multi.rad.1.use.vec.rad.ineq} is implied by \eqref{vec.rad.ineq}. In \eqref{supp.lem.multi.rad.1.drop.abs}, we drop the absolute value as we can make $\sum_{c=1}^C \epsilon_c (f_c(x) - g_c(x))$ non-negative by exchanging $f$ and $g$ for any realization of $\epsilon_c, c=1,\dots,C$.
\end{proof}

Now we move on to the next step. 
\begin{lemma} \label{multi.rad.lemma}
Let $N, C \in \mathbb{N}$. Let $\mathcal{H}$ be a set of functions from $\mathcal{X}$ to $\mathbb{R}^C$. $\forall i = 1, \dots, N$, let $w_i$ be a positive real numbers, and let $\lambda_i: \mathbb{R}^C \rightarrow \mathbb{R}$ a Lipschitz function. Denote the Lipschitz constant of $\lambda_i$ by $\abs{\lambda_i}$. Then, $$\mathbb{E}_{\epsilon_i} \sup_{f \in \mathcal{H}} \sum_{i=1}^N \epsilon_i w_i \lambda_i(f(x_i)) \leq \sqrt{2} \mathbb{E}_{\epsilon_{i, c}} \sup_{f \in \mathcal{H}} \sum_{i=1}^N w_i |\lambda_i| \sum_{c=1}^C \epsilon_{i,c} f_c(x_i) $$ 
where $\epsilon_i$'s and $\epsilon_{i,c}$'s are independent Rademacher variables and $f_c(x)$ denotes the c-th entry of $f(x)$.
\end{lemma}
\begin{proof}
Let $m = 0,1,\dots, N$. We 
prove
\begin{multline}
    \mathbb{E}_{\epsilon_i} \sup_{f \in \mathcal{H}} \sum_{i=1}^N {\epsilon_i w_i \lambda_i(f(x_i))} \leq \\
    \mathbb{E}_{\epsilon_{i, c}, \epsilon_i} \left[ \sup_{f\in \mathcal{H}} \sqrt{2} \sum_{1 \leq i \leq m} w_i |\lambda_i| \sum_{c=1}^C \epsilon_{i, c} f(x_i) + \sum_{m < i \leq N} \epsilon_i w_i \lambda_i(f(x_i)) \right]
\end{multline}
by induction on $m$.

The base case when $m=0$ holds with equality. The case when $m=N$ is the desired inequality. Now, suppose the inequality hold for $m-1$.

\begin{align}
    & \ \ \ \ \mathbb{E}_{\epsilon_i} \sup_{f \in \mathcal{H}} \sum_{i=1}^N \epsilon_i w_i \lambda_i(f(x_i)) \\
    & \leq  \mathbb{E}_{\epsilon_{i, c}, \epsilon_i} \left[ \sup_{f\in \mathcal{H}} \sqrt{2} \sum_{1 \leq i < m} w_i |\lambda_i| \sum_{c=1}^C \epsilon_{i, c} f_c(x_i) + \sum_{m \leq i \leq N} \epsilon_i w_i \lambda_i(f(x_i)) \right] \\
    & = \mathbb{E}_{\epsilon_{i, c}, \epsilon_i } \left[ \sup_{f\in \mathcal{H}}   \epsilon_m  w_m \lambda_m(f(x_m)) + \phi(f) \right] \\
    & = \mathbb{E}_{\{\epsilon_{i, c}, \epsilon_i| i \neq m\}} \mathbb{E}_{\epsilon_{m, c}} \left[ \sup_{f\in \mathcal{H}}   \epsilon_m  w_m \lambda_m(f(x_m)) + \phi(f) \right] \\
    & \leq \mathbb{E}_{\{\epsilon_{i, c}, \epsilon_i| i \neq m\}} \mathbb{E}_{\epsilon_{m, c}} \left[ \sup_{f\in \mathcal{H}}  \sqrt{2} w_m |\lambda_m| \sum_{c=1}^C \epsilon_{m, c} f_c(x_m) + \phi(f) \right] \\
    & = \mathbb{E}_{\epsilon_{i, c}, \epsilon_i} \left[ \sup_{f\in \mathcal{H}} \sqrt{2} \sum_{1 \leq i \leq m} w_i |\lambda_i| \sum_{c=1}^C \epsilon_{i, c} f_c(x_i) + \sum_{m < i \leq N} \epsilon_i w_i \lambda_i(f(x_i)) \right]
\end{align}
In the first equality, we let $\phi(f)$ denote the rest of the summation.
\end{proof}

\begin{lemma} \label{lemma: gen}
Let $\cG \subset \psi \circ \cF$ s.t. $\sup_{x\in X, g \in \cG} \norm{g\prs{x}}_2 \leq A$ for some constant $A$. Let $N\in\NN$ and $\cT = \cbs{T_i}_{i=1}^N$ be a sequence of invertible column-stochastic matrices. Fix $\prs{w_1, w_2, \dots, w_N}^{tr} \in \Delta^N$ and $n_{i,c} \in \NN$ for each $i\in\NN_N$ and $c\in\cY$. Let $S=\cbs{X_{i, c, j}: i = \NN_N, c \in \cY, j = \NN_{n_{i,c}}}$ where each $X_{i, c, j}$ is drawn from the class conditional distribution $P_{T_i}(\cdot\mid c)$ and all $X_{i, c, j}$'s are independent. $\forall i \in \NN_N$ and $c\in\cY$, let $\alpha_i \in \mathring{\Delta}^C$ $s.t.$ $\alpha_i(c)$ = $P_{T_i} (\tilde{Y}=c)$. Let $\ell$ be a proper loss s.t. $\forall i, c$ the function $\lambda_{\ell_{T_i}} \prs{\cdot ,c}$ is Lipschitz. Write 
$$\hat{\cR}_{w, S}(g) := \sum_{i=1}^N w_i \sum_{c=1}^C \frac{\alpha_i(c)}{n_{i, c}} \sum_{j=1}^{n_i} \lambda_{\ell_{T_i}} \prs{g\prs{X_{i,c,j}}, c)}$$ and $$\widetilde{\cR}(g) := \widetilde{R}_{\ell, P, \cT} \prs{\psi^{-1} \circ g} = \EE \brs{\hat{\cR}_{w, S}(g)}.$$ 
Then $\forall \delta \in (0, 1]$, with probability at least $1-\delta$ $w.r.t.$ to the draw of $S$,  
\begin{align}
\sup_{g\in\cG} \abs{\hat{\cR}_{w, S}(g) - \widetilde{\cR}(g)} 
\leq &\sqrt{2 \log \frac{2}{\delta} \sum_{i=1}^N \sum_{c=1}^C \frac{w_i^2 \alpha_i^2(c)}{n_{i, c}} \prs{\abs{\lambda_{\ell_{T_i}}}A + \abs{\lambda_{\ell_{T_i}}}_0}^2} 
\\
& +2 \EE_{S, \epsilon_{i,c,j, c'}}\brs{\sup_{g\in\cG} \sum_{i=1}^N w_i \abs{\lambda_{\ell_{T_i}}} \sum_{c=1}^C \frac{\alpha_i(c)}{n_{i,c}} \sum_{j=1}^{n_{i,c}}  \sum_{c'=1}^C \epsilon_{i,c,j,c'}  g_{c'}(X_{i, c,j})},
\end{align} where $\epsilon_{i,c, j, c'}, i \in \NN_N, c\in\cY, c'\in\cY, j \in \NN_{n_{i,c}}$ are i.i.d. Rademacher random variables, $\abs{\lambda_{\ell_{T_i}}}_0 = \max_{c} \abs{\lambda_{\ell_{T_i}}(0, c)}$, and $\abs{\lambda_{\ell_{T_i}}}$ is the smallest real number such that it is a Lipschitz constant of $\lambda_{\ell_{T_i}} \prs{\cdot ,c}$, $\forall i, c$.
\end{lemma}
\begin{proof}[Proof of Theorem \ref{lemma: gen}]
The theorem is a direct result of Lemmas \ref{rad.lips.lemma} and \ref{multi.rad.lemma}.
\end{proof}

In Theorem \ref{llp.cali}, we saw that $\widetilde{\cR}(g)$ is a risk for LMNTM satisfying an excess risk bound. Lemma \ref{lemma: gen} shows that $\hat{\cR}_{w, S}(g)$ is an accurate estimate of $\widetilde{\cR}(g)$, and therefore justifies its use as an empirical objective for LMNTM.

The second term on the right hand side of the inequality in Lemma \ref{lemma: gen} depends on the choice of hypothesis class $\cG$, and can be viewed as a generalization of Rademacher complexity to LMNTM. To make this term more concrete, we study two popular choices of function classes, the reproducing kernel Hilbert space (RKHS) and the multilayer perceptron (MLP). We first consider the kernel class.

\begin{proposition} \label{prop: kernel_bound}
Let $k$ be a symmetric positive definite (SPD) kernel, and let $\mathcal{H}$ be the associated reproducing kernel Hilbert space (RKHS). Assume $k$ bounded by $K$, meaning $\forall x$, $\norm{k(\cdot, x)}_{\cH} \leq K$. Let $\cG^k_{K,R}$ denote the ball of radius R in $\cH$ and $\cG = \underbrace{\cG^k_{K,R} \times \cG^k_{K,R} \times \dots \times \cG^k_{K,R}}_{\text{the Cartesian product of $C$ $\cG^k_{K,R}$'s}}$. Then 
\begin{multline}
    \EE_{S, \epsilon_{i,c,j, c'}} \brs{\sup_{g\in \cG} \sum_{i=1}^N {w_i} \abs{\lambda_{\ell_{T_i}}}\sum_{c=1}^C \frac{\alpha_i(c)}{n_{i,c}}  \sum_{j=1}^{n_{i,c}} \sum_{c'=1}^C \epsilon_{i, c, j, c'}g_{c'}(X_{i, c, j})} \leq
    \\
    CRK\sqrt{\sum_{i=1}^N {w_i^2 \abs{\lambda_{\ell_{T_i}}}^2} \sum_{c=1}^C\frac{\alpha^2_i(c)}{n_{i,c}}} ,
\end{multline}
where $\epsilon_{i,c, j, c'}, i \in \NN_N, c\in\cY, c'\in\cY, j \in \NN_{n_{i,c}}$ are i.i.d. Rademacher random variables.
Thus the generalization error bound becomes:  $\forall \delta \in \brs{0, 1}$, with probability at least $1-\delta$, 
\begin{multline}
    \sup_{g\in\cG} \abs{\hat{\cR}_{w, S}(g) - \widetilde{\cR}(g)} \leq 
    \\
    \prs{ \max_i \prs{\abs{\lambda_{\ell_{T_i}}}A + \abs{\lambda_{\ell_{T_i}}}_0} \sqrt{2 \log \frac{2}{\delta}} + CRK \max_i\abs{\lambda_{\ell_{T_i}}}} \sqrt{\sum_{i=1}^N {w_i^2} \sum_{c=1}^C\frac{\alpha^2_i(c)}{n_{i,c}}}.
\end{multline}
\end{proposition}

\emph{Proof of Proposition~\ref{prop: kernel_bound}.} For the reader's convenience, we restate the result:

\begin{proposition} \label{supp.prop: kernel_bound}
Let k be a symmetric positive definite (SPD) kernel bounded by $K$ and $\mathcal{H}$ be the associated reproducing kernel Hilbert space (RKHS). $i.e.\norm{k(\cdot, x)}_{\cH} \leq K$.  Let $\cG^k_{K,R}$ denote the ball of radius R in $\cH$ and $\cG = \underbrace{\cG^k_{K,R} \times \cG^k_{K,R} \times \dots \times \cG^k_{K,R}}_{\text{the Cartesian products of $C$ $\cG^k_{K,R}$'s}}$. Then 
$$\EE_{\epsilon_{i,c}} \brs{\sup_{g_c \in \cG^k_{K,R}} \sum_{i=1}^M a_i \sum_{c=1}^C \epsilon_{i,c} g_c(x_i) } \leq CRK\sqrt{\sum_{i=1}^M a_i^2}$$
where $a_i > 0$, and $\epsilon_{i,c}$ are independent Rademacher random variables.
\end{proposition}
\begin{proof}
First, by Cauchy-Schwartz inequality, observe $\forall R > 0, g \in \cG^k_{K,R}, x \in \mathcal{X}$ $$|g(x)| = |\langle g, k(\cdot, x) \rangle| \leq \|g\|_{\mathcal{H}}\|k(\cdot, x)\|_{\mathcal{H}} \leq RK.$$ Thus, 
\begingroup
\allowdisplaybreaks
\begin{align}
    & \EE_{\epsilon_{i,c}} \left[\sup_{g_c \in \cG^k_{K,R}} \sum_{i=1}^M a_i \sum_{c=1}^C \epsilon_{i,c} g_c(x_i) \right] \\
    & = \EE_{\epsilon_{i,c}} \left[\sup_{g_c \in \cG^k_{K,R}} \sum_{i=1}^M a_i \sum_{c=1}^C \epsilon_{i,c} \langle g_c, k(\cdot, x_i) \rangle \right] \label{supp.eq: repro prop} \\
    & = \EE_{\epsilon_{i,c}} \left[\sup_{g_c \in \cG^k_{K,R}}  \sum_{c=1}^C \langle g_c, \sum_{i=1}^M a_i \epsilon_{i,c} k(\cdot, x_i) \rangle \right]  \\
    & = \EE_{\epsilon_{i,c}} \left[\sum_{c=1}^C \langle R \frac{\sum_{i=1}^M a_i \epsilon_{i,c}k(\cdot, x_i)}{\|\sum_{i=1}^M a_i \epsilon_{i,c}k(\cdot, x_i)\|}, \sum_{i=1}^M a_i \epsilon_{i,c}k(\cdot, x_i)\rangle \right] \label{supp.eq: CS} \\
    & = R\sum_{c=1}^C \EE_{\epsilon_{i,c}} \sqrt{ \left \| \sum_{i=1}^M a_i \epsilon_{i,c} k(\cdot, x_i) \right\|^2} \\
    & \leq R \sum_{c=1}^C \sqrt{\EE_{\epsilon_{i,c}} \left\| \sum_{i=1}^M a_i \epsilon_{i,c} k(\cdot, x_i) \right\|^2} \label{supp.inq: Censen} \\
    & = CR \sqrt{\sum_{i=1}^M a_i^2 \left\|k(\cdot, x_i) \right\|^2} \label{supp.eq.kernel: ind of rads r.v.} \\
    & = CRK \sqrt{\sum_{i=1}^M a_i^2}
\end{align}
\endgroup
Equality \eqref{supp.eq: repro prop} and \eqref{supp.eq: CS} follow the reproducing property and the equality condition of Cauchy-Schwarz, respectively. \eqref{supp.inq: Censen} is implied by Jensen's inequality and \eqref{supp.eq.kernel: ind of rads r.v.} by the independence of Rademacher random variables.
\end{proof}

We now define the Rademacher Complexity-like term $\EE_{\epsilon_i} \sup_{g\in \cG} \sum_{i=1}^M a_i \epsilon_i g(x_i)$ formally and characterize several properties which will be used in the proof of Proposition \ref{prop: MLP_bound}.
\newcommand{\rad}{\text{Rad}_{S, a}}
\begin{definition} Let $\cG$ be a subset of measurable functions from $\cX$ to $\RR$. Denote the sample path $S=\prs{x_i}_{i=1}^M$ and weights by $a=\prs{a_i}_{i=1}^M$ where $a_i \geq 0$. Define
$$\rad (\cG) = \EE_{\epsilon_i} \sup_{g\in \cG} \sum_{i=1}^M a_i \epsilon_i g(x_i),$$ where $\epsilon_i$'s are i.i.d. Rademacher random variables.  
\end{definition}

\begin{proposition} \label{supp.prop.weighted_rad_properties} $\rad$ has the following properties:
\begin{enumerate}
  \item $\cG \subset \cH \implies \rad(\cG) \leq \rad(\cH)$ \label{supp.prop.rad: subset}
  \item $\rad(\cG_1 + \cG_2) = \rad(\cG_1) + \rad(\cG_2)$, \\ where $\cG_1 + \cG_2 = \cbs{g_1 + g_2: g_1 \in \cG_1, g_2 \in \cG_2}$ \label{supp.prop.rad: sum}
  \item $\forall c_0 \in \RR, \rad(c_0 \cG) = \abs{c_0} \rad(\cG)$, where $c_0 \cG := \cbs{c_0 g: g\in \cG}$ \label{supp.prop.rad: mult}
  \item $\rad(\operatorname{conv}\cG) = \rad(\cG)$, where $\operatorname{conv}\cG$ denotes the convex hull of $\cG$. \label{supp.prop.rad: conv}
  \item Let $\mu: \RR \rightarrow \RR$ be a Lipschitz function and let $\abs{\mu}$ be its Lipschitz constant. Then, $$\rad(\mu \circ \cG) \leq \abs{\mu} \rad(\cG), \text{ where } \mu \circ \cG = \cbs{\mu \circ g: g\in \cG}.$$ \label{supp.prop.rad: lip}
\end{enumerate}
\end{proposition}
\begin{proof}
Property \ref{supp.prop.rad: subset} and \ref{supp.prop.rad: sum} immediately follow the definition. Property \ref{supp.prop.rad: mult} is implied by the invariance of $\epsilon_i$ under negation. It remains to prove Property \ref{supp.prop.rad: conv} and \ref{supp.prop.rad: lip}.

For Property \ref{supp.prop.rad: conv}:
\begingroup \allowdisplaybreaks \begin{align}
    & \rad(\operatorname{conv} \cG) \\
    & = \EE \sup_{n\in\NN} \sup_{\lambda \in \Delta^n, g_j \in \cG} \sum_{i=1}^M a_i \epsilon_i \sum_{j=1}^n \lambda_j g_j(x_i)\\
    & = \EE \sup_{n\in\NN} \sup_{\lambda \in \Delta^n, g_j \in \cG} \sum_{j=1}^n \lambda_j \sum_{i=1}^M a_i \epsilon_i  g_j(x_i)\\
     & = \EE \sup_{n\in\NN} \sup_{\lambda \in \Delta^n, g_j \in \cG} \max_{j} \sum_{i=1}^M a_i \epsilon_i  g_j(x_i)\\
     & = \EE \sup_{ g \in \cG} \sum_{i=1}^M a_i \epsilon_i  g(x_i) \\
     & =  \rad(\cG).
\end{align} \endgroup

For Property \ref{supp.prop.rad: lip}, we follow the idea of \textcite{meir_zhang_03_rad},
\begingroup \allowdisplaybreaks \begin{align}
    & \rad(\mu \circ \cG) \\ 
    & = \EE_{\epsilon_i} \sup_{g\in \cG} \sum_{i=1}^M a_i \epsilon_i (\mu \circ g)(x_i) \\
    & = \EE_{\epsilon_i, i= 2, 3, \dots, M} \EE_{\epsilon_1} \sup_{g\in \cG} \sum_{i=1}^M a_i \epsilon_i (\mu \circ g)(x_i) \\
    & = \frac{1}{2} \EE_{\epsilon_i, i= 2, 3, \dots, M} \left[\sup_{g\in \cG} \prs{ a_1 (\mu \circ g)(x_1) + \sum_{i=2}^M a_i \epsilon_i (\mu \circ g)(x_i)} \right. \nonumber \\  
    & \left. +  \sup_{g'\in \cG} \prs{- a_1 (\mu \circ g')(x_1) + \sum_{i=2}^M a_i \epsilon_i (\mu \circ g')(x_i)} \right] \\
    & = \frac{1}{2} \EE_{\epsilon_i, i= 2, 3, \dots, M} \brs{\sup_{g, g'\in \cG}  a_1 \prs{ \mu(g(x_1)) - \mu(g'(x_1))} + \sum_{i=2}^M a_i \epsilon_i (\mu \circ (g+g'))(x_i)} \\
    & \leq \frac{1}{2} \EE_{\epsilon_i, i= 2, 3, \dots, M} \brs{\sup_{g, g'\in \cG} a_1 \abs{\mu}\abs{g(x_1) - g'(x_1)} + \sum_{i=2}^M a_i \epsilon_i (\mu \circ (g+g'))(x_i)} \\
    & = \frac{1}{2} \EE_{\epsilon_i, i= 2, 3, \dots, M} \brs{\sup_{g, g'\in \cG} a_1 \abs{\mu}\prs{g(x_1) - g'(x_1)} + \sum_{i=2}^M a_i \epsilon_i (\mu \circ (g+g'))(x_i)} \label{supp.prop.rad: lip.dropabs}\\
    & = \frac{1}{2} \EE_{\epsilon_i, i= 2, 3, \dots, M} \left[ \sup_{g \in \cG} \prs{a_1 \abs{\mu}g(x_1) + \sum_{i=2}^M a_i \epsilon_i (\mu \circ g)(x_i)} \right. \nonumber \\ 
    & \left. + \sup_{g'\in\cG} \prs{-a_1 \abs{\mu}g'(x_1) + \sum_{i=2}^M a_i \epsilon_i (\mu \circ g')(x_i)} \right]\\
    & = \EE_{\epsilon_i} \sup_{g\in\cG} \brs{a_1 \abs{\mu} g(x_1) \epsilon_1 +  \sum_{i=2}^M a_i \epsilon_i (\mu \circ g)(x_i)} .
\end{align} \endgroup
In step \eqref{supp.prop.rad: lip.dropabs}, we can drop the absolute value since we can always make $\prs{g(x_1) - g'(x_1)}$ non-negative by exchanging $g$ and $g'$ while leaving the rest of the equation invariant. Proceeding by the above argument inductively on $i$, we eventually have $$\rad(\mu \circ \cG) \leq \EE_{\epsilon_i} \sup_{g\in\cG}\sum_{i=1}^M a_i \abs{\mu} g(x_i) \epsilon_i = \abs{\mu} \rad(\cG)$$
as desired.
\end{proof}

To simplify the notations, we follow \textcite{zhang2017radnn} and define the real-valued MLP inductively:
$$\cN_1 = \cbs{x\rightarrow \innerprod{x, v}: v\in \RR^d, \norm{v}_2 \leq \beta},$$
$$\cN_m = \cbs{x\rightarrow \sum_{j=1}^d w_j \mu(f_j(x)): v\in \RR^d, \norm{v}_1 \leq \beta, f_j \in \cN_{m-1}},$$
where $\beta \in \RR_+$ and $\mu$ is a $1$-Lipschitz activation function. Define an MLP which outputs a vector in $\RR^C$ by $\cG = \underbrace{\cN_m \times \cN_m \times \dots \times \cN_m}_{\text{the Cartesian product of $C$ $\cN_m$'s}}$. To leverage standard techniques for the proof, we additionally assume $\forall m \in \NN, 0 \in \mu \circ \cN_m$.

\begin{proposition} \label{prop: MLP_bound}
Let $\cG = \underbrace{\cN_m \times \cN_m \times \dots \times \cN_m}_{\text{the Cartesian product of $C$ $\cN_m$'s}}$. Assume $\forall x \in \cX, \norm{x_i} \leq \alpha$ and $\forall m \in \NN, 0 \in \mu \circ \cN_m$. Then, 
\begin{multline}
    \EE_{S, \epsilon_{i,c, j, c'}} \brs{\sup_{g\in \cG} \sum_{i=1}^N {w_i} \abs{\lambda_{\ell_{T_i}}}\sum_{c=1}^C \frac{\alpha_i(c)}{n_{i,c}}  \sum_{j=1}^{n_{i,c}} \sum_{c'=1}^C \epsilon_{i, c, j, c'}g_{c'}(X_{i, c, j})} \leq
    \\
    C\alpha2^{m-1}\beta^m\sqrt{\sum_{i=1}^N {w_i^2 \abs{\lambda_{\ell_{T_i}}}^2} \sum_{c=1}^C\frac{\alpha^2_i(c)}{n_{i,c}}},
\end{multline}
where $\epsilon_{i,c, j, c'}, i \in \NN_N, c\in\cY, c'\in\cY, j \in \NN_{n_{i,c}}$ are i.i.d. Rademacher random variables. Thus, the generalization error bound becomes:  $\forall \delta \in \brs{0, 1}$, with probability at least $1-\delta$,  
\begin{align}
    &\sup_{g\in\cG} \abs{\hat{\cR}_{w, S}(g) - \widetilde{\cR}(g)} \\ 
    &\leq \prs{ \max_i \prs{\abs{\lambda_{\ell_{T_i}}}A + \abs{\lambda_{\ell_{T_i}}}_0} \sqrt{2 \log \frac{2}{\delta}}  + C\alpha2^{m-1}\beta^m \max_i\abs{\lambda_{\ell_{T_i}}} }  \sqrt{\sum_{i=1}^N {w_i^2} \sum_{c=1}^C\frac{\alpha^2_i(c)}{n_{i,c}}}. 
\end{align}

\end{proposition}

\emph{Proof of Proposition~\ref{prop: MLP_bound}.} For the reader's convenience, we restate the result:

\begin{proposition} \label{supp.prop: MLP_bound}
Let $\cG = \underbrace{\cN_m \times \cN_m \times \dots \times \cN_m}_{\text{the Cartesian products of $C$ $\cN_m$'s}}$. Assume $\forall x \in \cX, \norm{x_i} \leq \alpha$ and $\forall k \in \NN, 0 \in \mu \circ \cN_k$. Then, 
$$\EE_{\epsilon_{i,c}} \brs{\sup_{g_c \in  \cN_m} \sum_{i=1}^M a_i \sum_{c=1}^C \epsilon_{i,c} g_c(x_i) } \leq C\alpha2^{m-1}\beta^m\sqrt{\sum_{i=1}^M a_i^2}.$$
where $a_i > 0$, and $\epsilon_{i,c}$ are independent Rademacher random variables.
\end{proposition}
Recall that the MLP outputs a vector in $\RR^C$. The set of MLPs is  $\cG = \underbrace{\cN_m \times \cN_m \times \dots \times \cN_m}_{\text{the Cartesian products of $C$ $\cN_m$'s}}$ where the set $\cN_m$ is defined inductively as
$$\cN_1 = \cbs{x\rightarrow \innerprod{x, v}: v\in \RR^d, \norm{v}_2 \leq \beta} \quad \mbox{for $m=1$, and }$$
$$\cN_m = \cbs{x\rightarrow \sum_{j=1}^d w_j \mu(f_j(x)): v\in \RR^d, \norm{v}_1 \leq \beta, f_j \in \cN_{m-1}} \quad \mbox{for $m>1$.}$$
$\beta \in \RR_+$, and $\mu$ is a $1$-Lipschitz activation function. We now proceed with the proof of Proposition~\ref{supp.prop: MLP_bound}.
\begin{proof}
We have
\begingroup \allowdisplaybreaks \begin{align}
    & \EE_{\epsilon_{i,c}} \brs{\sup_{g_c \in \cN_m} \sum_{i=1}^M a_i \sum_{c=1}^C \epsilon_{i,c} g_c(x_i)}  \\
    & = \EE_{\epsilon_{i,c}} \brs{\sup_{g_c \in \cN_m} \sum_{c=1}^C \sum_{i=1}^M a_i  \epsilon_{i,c} g_c(x_i)}  \\
    & \leq \sum_{c=1}^C \EE_{\epsilon_{i,c}} \brs{\sup_{g_c \in \cN_m} \sum_{i=1}^M a_i  \epsilon_{i,c} g_c(x_i)}  \\
    & = C \EE_{\epsilon_{i}} \brs{\sup_{g \in \cN_m} \sum_{i=1}^M a_i  \epsilon_{i} g(x_i)} \\
    & = C \rad(\cN_m) 
\end{align} \endgroup
where
$$ \rad(\cN_m) =  \EE_{\epsilon_{i}} \brs{\sup_{h_j \in \cN_{m-1}, \norm{v}_1 \leq \beta} \sum_{i=1}^M a_i  \epsilon_{i} \prs{ \sum_{j=1}^d v_j (\mu \circ h_j)}(x_i)} \label{supp.nn.eq v_eps}.$$
Note $\sum_{j=1}^d v_j (\mu \circ h_j) \in \beta \operatorname{conv}\prs{\mu\circ\cN_{m-1} - \mu\circ\cN_{m-1}}$. Here the difference between two sets of functions is $\cG_1 - \cG_2 = \cbs{g_1 - g_2: g_1 \in \cG_1, g_2 \in \cG_2}$ and $\beta\cG_1 = \cbs{\beta g_1: g_1\in\cG_1}$ for a real number $\beta$. Apply Proposition \ref{supp.prop.weighted_rad_properties},

\begingroup \allowdisplaybreaks \begin{align}
    & \rad(\cN_m) \\
    & \leq \rad(\beta \operatorname{conv}\prs{\mu\circ\cN_{m-1} - \mu\circ\cN_{m-1}}) \\
    & = \beta \rad(\operatorname{conv}\prs{\mu\circ\cN_{m-1} - \mu\circ\cN_{m-1}}) \\
    & = \beta \rad(\prs{\mu\circ\cN_{m-1} - \mu\circ\cN_{m-1}}) \\
    & = \beta \prs{\rad(\mu\circ\cN_{m-1}) + \rad( - \mu\circ\cN_{m-1})} \\
    & = 2 \beta \rad(\mu\circ\cN_{m-1}) \\ 
    & \leq 2 \abs{\mu}\beta \rad(\cN_{m-1})
\end{align} \endgroup
Proceeding backward inductively on $m$, we have $\rad(\cN_m) \leq 2^{m-1} \beta^{m-1} \rad(\cN_1)$. The set $\cN_1$ can be viewed as the ball with radius $\beta$ centered at $0$ in the RKHS associated to linear kernel bounded $\alpha$, so we can apply Proposition \ref{supp.prop: kernel_bound}. Therefore, $$\rad(\cN_m) \leq 2^{m-1} \beta^{m-1} \rad(\cN_1) \leq  2^{m-1} \beta^{m} \alpha \sqrt{\sum_{i=1}^M a_i^2} $$ and $$\EE_{\epsilon_{i,c}} \brs{\sup_{g_c \in  \cN_m} \sum_{i=1}^M a_i \sum_{c=1}^C \epsilon_{i,c} g_c(x_i) } \leq C \rad(\cN_m) \leq C\alpha2^{m-1}\beta^m\sqrt{\sum_{i=1}^M a_i^2}$$
as desired.
\end{proof}

\begin{proof}[Proof of Theorem \ref{thm: geb}]
Theorem \ref{thm: geb} follows Lemma \ref{lemma: gen}, Proposition \ref{prop: kernel_bound}, Proposition \ref{prop: MLP_bound}, and the fact that $\alpha_i(c) \leq 1$.
\end{proof}

\subsection{Proof of Proposition \ref{lip.upper.bd.cel}}
\begin{proof}[Proof of Proposition \ref{lip.upper.bd.cel}]
By Corollary 1.42 of \textcite{Weaver99Lip}, $\norm{\norm{\nabla_s \lambda_{\ell_{T}}(s, y)}_2}_{\infty}$ is a Lipschitz constant of $\lambda_{\ell_{T}}(\cdot, y)$, where $y\in\cbs{1, 2 \dots, C}$, $\nabla$ denotes the gradient of a function, $\norm{\nabla_s \lambda_{\ell_{T}}(s, y)}_2$ is a function maps $s$ to a real number, and the $\norm{\cdot}_{\infty}$ takes the essential supremum over $\Delta^C$. We use $t_{i,j}$ to denote the element at $i$-row and $j$-column of $T$.

$$\lambda_{\ell_{T}}(s, y) = -\log\prs{\sum_{k=1}^C t_{y, k} \frac{e^{s_k}}{\sum_{j=1}^C e^{s_j}}} = -\log \prs{\sum_{k=1}^C t_{y, k} e^{s_k}} + \log\prs{\sum_{j=1}^C e^{s_j}}.$$

\begin{align}
    \frac{\partial \lambda_{\ell_{T}}(s, y)}{\partial s_i} & = -\frac{t_{y, i} e^{s_i}}{\sum_{j=1}^C t_{y, j}e^{s_j}} + \frac{e^{s_i}}{\sum_{j=1}^C e^{s_j}} = -\frac{t_{y, i} \frac{e^{s_i}}{\sum_{k=1}^C e^{s_k}}}{\sum_{j=1}^C t_{y, j}\frac{e^{s_j}}{\sum_{j=k}^C e^{s_k}}} + \frac{e^{s_i}}{\sum_{j=1}^C e^{s_j}} \\
    & = -\frac{t_{y, i} p_i}{\sum_{j=1}^C t_{y, j}p_j} +p_i
\end{align}
In the last equality, we denote $\frac{e^{s_i}}{\sum_{k=1}^C e^{s_k}}$ by $p_i$.
Then, 
\begin{align}
    \norm{\nabla_s \lambda_{\ell_{T}}(s, y)}_2^2 & = \sum_{i=1}^C \prs{-\frac{t_{y, i} p_i}{\sum_{j=1}^C t_{y, j}p_j} +p_i}^2 \leq \sum_{i=1}^C \abs{-\frac{t_{y, i} p_i}{\sum_{j=1}^C t_{y, j}p_j} +p_i} \label{supp.lip.remove.sq} \\
    &\leq \sum_{i=1}^C \prs{\frac{t_{y, i} p_i}{\sum_{j=1}^C t_{y, j}p_j} +p_i} = 2
\end{align}
The inequality in step \eqref{supp.lip.remove.sq} follows the observation that $\abs{-\frac{t_{y, i} p_i}{\sum_{j=1}^C t_{y, j}p_j} +p_i} \leq 1$
\end{proof}

\section{Confirmation of Probabilistic Model}
\label{sec:cfm prob model}

In Section \ref{subsec: prob model c bags}, we state that $\alpha\prs{i} = \Bar{P}_T(\Tilde{Y} = i)$, $P_{\gamma_i}\prs{\cdot}=\Bar{P}_T(\cdot\mid\Tilde{Y} = i)$, and $\gamma_i\prs{c}=\Bar{P}_T(Y=c\mid\Tilde{Y} = i)$ for matrix $T$ with $T\prs{i,j} = \frac{\gamma_{i}\prs{j}\alpha\prs{i}}{\sigma\prs{j}}$. Here we confirm these facts. 
 
 Let $T$ be a stochastic matrix with entries $T\prs{i,j} = \frac{\gamma_{i}\prs{j}\alpha\prs{i}}{\sigma\prs{j}}$. We construct the joint probability measure $\Bar{P}_T$ on $\cX \times \cY \times \cY$  as described in Section \ref{llnfc}. We can see $\Bar{P}_T\prs{\Tilde{Y}=i}=\sum_{j=1}^C \Bar{P}_T\prs{\Tilde{Y}=i, Y=j} = \sum_{j=1}^C \Bar{P}_T\prs{Y=j} T\prs{i,j} = \sum_{j=1}^C \sigma\prs{j} \frac{\gamma_i\prs{j}\alpha\prs{i}}{\sigma\prs{j}} = \alpha\prs{i}$ and $\forall \text{ events } \cA \subset \cX,  \forall i, y \in \cY$
$\Bar{P}_T\prs{\Tilde{Y}=i}=\alpha_i$ and $\forall \cA \in \cM_{\cX},  \forall i, y \in \cY$
 \begin{align}
   & \Bar{P}_T\prs{X\in\cA, Y=y \mid \Tilde{Y}=i} \\
   & = \frac{1}{\alpha\prs{i}}{\Bar{P}_T\prs{X\in\cA , Y=y, \Tilde{Y}=i}} \\
   & = \frac{1}{\alpha\prs{i}}{ P\prs{X\in\cA , Y=y} \frac{\gamma_i\prs{y}\alpha\prs{i}}{\sigma\prs{y}}} \\
   & = { P_y\prs{X\in\cA} \gamma_i\prs{y}}\\
   & = P_{\gamma_i} \prs{X\in\cA, Y=y}.
 \end{align}

 Hence, $\Bar{P}_T\prs{\cdot \mid \Tilde{Y}=i} = P_{\gamma_i} \prs{\cdot}$, which implies that $\Bar{P}_T\prs{Y=c \mid \Tilde{Y}=i} = P_{\gamma_i} \prs{Y=c} = \gamma_i\prs{c}$, and for a data point $\prs{X,Y,\Tilde{Y}} \sim \Bar{P}_T$ the event $\Tilde{Y}=i$ entails that $\prs{X,Y} \sim P_{\gamma_i}$.

\section{Grouping and Weights Optimization}
\label{suppsec:weights}
To optimize the weights or the assignment of bags we would need to optimize the composition of our two bounds: $\theta(\mathcal{R}_{L, P}(f)  - \mathcal{R}_{L, P}^*) \leq \text{Emprical Risk} + \text{Generalization Error Bound} - \mathcal{R}^*_{l, P, \mathcal{T}}$. This is in contrast to the approach with backward correction \cite{binary} which does not require the excess risk bound (because their excess target risk is simply proportional to the excess surrogate risk). Therefore, to optimize the composition of our bounds, we'd need to estimate the surrogate Bayes risk, a challenging task. We also note that both the generalization error bound and excess risk bound involve weights $w_i$ and noise matrices $T_i$. Therefore, even if the surrogate Bayes risk were somehow known, the resulting integer programming problem is much more involved than for the backward correction, where it's a simple matching problem. 

Fortunately, LLPFC with random partitioning and weights which optimize solely generalization error bound yields superior empirical results in the experiments and outperforms other multiclass LLP methods by a significant margin. We believe weight optimization is much more important for the backward correction, where the loss functions can have large and disparate magnitudes (which need to be offset by carefully chosen weights), than it is for forward correction where the outputs of the inverse link function are in the unit simplex and thus all of a comparable magnitude. A similar point is made by \textcite{Patrini2017MakingDN} in the last two sentence in the first paragraph of section 6.


\end{document}